\pdfoutput=1
\documentclass{article}%
\usepackage{arxiv}
\usepackage{times}
\usepackage{natbib}
\usepackage{xcolor}
\usepackage{marvosym}
\usepackage[bottom]{footmisc}
\date{}
\renewcommand{\undertitle}{}
\renewcommand{\headeright}{}
\usepackage{etoolbox}
\makeatletter\patchcmd{\@maketitle}{\vskip 0.4in \@minus 0.1in \center{\@date}   \vskip 0.2in}{\vskip 0.02in}{}{\typeout{PATCH FAILED}}\makeatother
\makeatletter\patchcmd{\@maketitle}{\textsc{\undertitle}\\
    \vskip 0.1in}{}{}{\typeout{PATCH2 FAILED}}\makeatother
\usepackage{amsmath,amsfonts,bm}

\def\eqref#1{equation~\ref{#1}}

\def\1{\bm{1}}

\DeclareMathAlphabet{\mathsfit}{\encodingdefault}{\sfdefault}{m}{sl}
\SetMathAlphabet{\mathsfit}{bold}{\encodingdefault}{\sfdefault}{bx}{n}

\renewcommand{\eqref}[1]{(\ref{#1})}%

\usepackage{hyperref}
\definecolor{linkblue}{RGB}{31,78,140}
\definecolor{linkgreen}{RGB}{58,120,70}
\hypersetup{colorlinks=true,linkcolor=linkgreen,citecolor=linkblue,urlcolor=linkgreen,
  pdftitle={Lucid Dreaming for World Models: Learning to Doubt Imagination and Decide by Trust},
  pdfauthor={Ziqi Wen, Ting Xu, Lianyu Wang, Xian Lin, Yanda Meng, Huazhu Fu, Meng Wang, Ching-Yu Cheng}}
\usepackage{microtype}
\makeatletter\renewcommand{\paragraph}{\@startsection{paragraph}{4}{\z@}{0pt}{-1em}{\normalsize\bfseries}}\makeatother
\usepackage{url}
\usepackage{needspace}
\usepackage{graphicx}
\usepackage{caption}
\usepackage{subcaption}
\usepackage{placeins}
\usepackage{wrapfig}
\makeatletter\def\WFadj{1.1}\patchcmd{\WF@startwrapping}{\advance\WF@size1.1\baselineskip}{\advance\WF@size\WFadj\baselineskip}{}{\errmessage{WFpatch failed}}\makeatother
\newcommand{\needfit}[1]{\par\ifdim\dimexpr\pagegoal-\pagetotal\relax<\dimexpr#1\relax\newpage\fi}
\makeatletter\newcommand{\WFfinish}[1]{\par\ifnum\c@WF@wrappedlines>\z@\vspace{\dimexpr\c@WF@wrappedlines\baselineskip\relax}\fi\WFclear}\makeatother
 
\usepackage{booktabs}
\usepackage{amsmath,amssymb}
\usepackage{amsthm}
\usepackage{placeins}
\usepackage{float}

\newtheoremstyle{compact}{2pt}{2pt}{}{}{\bfseries}{.}{ }{}
\theoremstyle{compact}%
\newtheorem{theorem}{Theorem}
\newtheorem{proposition}{Proposition}
\newtheorem{corollary}{Corollary}
\newtheorem{lemma}{Lemma}

\usepackage[most]{tcolorbox}
\usepackage{multirow}
\usepackage{xltabular}
\usepackage{colortbl}%
\newtcolorbox{keybox}{enhanced,colback=black!5,colframe=black!28,boxrule=0.4pt,arc=1.5pt,left=7.6pt,right=7.6pt,top=5pt,bottom=5pt,boxsep=0pt,before skip=12.0pt plus 2pt,after skip=10.0pt plus 2pt,fontupper=\normalsize}
\definecolor{lucid}{RGB}{122,60,110}
\definecolor{lucidlight}{RGB}{244,236,242}
\definecolor{dirok}{RGB}{214,229,217}%
\definecolor{dirbad}{RGB}{243,218,214}%
\definecolor{wrongway}{RGB}{230,230,230}%
\definecolor{defgreen}{RGB}{58,120,70}
\definecolor{defgreenlight}{RGB}{238,246,238}
\newtcolorbox{defbox}{enhanced,breakable,colback=defgreenlight,colframe=defgreen!50,boxrule=0.4pt,arc=1.5pt,left=5pt,right=5pt,top=3pt,bottom=3pt,boxsep=1pt,before skip=6pt,after skip=6pt}
\tcbset{appstatement/.style={enhanced,colframe=black!28,boxrule=0.4pt,arc=1.5pt,
    left=7.6pt,right=7.6pt,top=5pt,bottom=5pt,boxsep=0pt,
    before skip=12pt plus 2pt,after skip=4pt}}
\newcommand{\appendixtypography}{%
  \renewtcolorbox{keybox}{appstatement,colback=black!5}%
  \tcolorboxenvironment{lemma}{appstatement,colback=white}%
  \tcolorboxenvironment{proof}{blanker,breakable,left=8pt,
    borderline west={0.8pt}{0pt}{black!38},
    before skip=4pt,after skip=10pt plus 2pt,
    before upper={\setlength{\parskip}{4pt plus 1pt}}}%
  \renewtcolorbox{defbox}{appstatement,colback=defgreenlight,colframe=defgreen!50,before skip=10pt plus 2pt,after skip=8pt}%
}
\makeatletter\g@addto@macro\appendix{\appendixtypography}\makeatother
\usepackage{etoc}
\definecolor{tocgray}{RGB}{118,118,118}
\newcommand{\apptocgroup}[1]{}
\newcommand{\appgroup}[1]{\addtocontents{toc}{\protect\apptocgroup{#1}}}
\newcommand{\apptocdots}[1]{\nobreak\leaders\hbox to 0.62em{\hss\textcolor{#1}{.}\hss}\hfill\nobreak}
\newcommand{\appendixcontents}{%
  \begingroup
  \hypersetup{linkcolor=black}\setlength{\parskip}{0pt}%
  \etocsettagdepth{main}{none}\etocsettagdepth{app}{subsection}%
  \renewcommand{\apptocgroup}[1]{\par\addvspace{12pt}\noindent
    {\footnotesize\scshape\textcolor{lucid}{##1}}\hspace{0.7em}%
    \textcolor{lucid!30}{\leaders\hrule height 2.9pt depth -2.5pt\hfill}\par\nobreak\vspace{3pt}}%
  \etocsettocstyle{%
    \noindent{\large\scshape Appendix}\hfill\raisebox{0.5pt}{\footnotesize\scshape\textcolor{tocgray}{page}}\par
    \vspace{3pt}{\color{lucid}\hrule height 0.8pt}\vspace{1pt}}%
    {\par\vspace{7pt}{\color{lucid!40}\hrule height 0.4pt}\vspace{14pt}}%
  \etocsetstyle{section}{}{\par\addvspace{6pt}}%
    {\noindent\makebox[1.8em][l]{\textcolor{lucid}{\bfseries\etocthenumber}}\etoclink{\scshape\etocname}%
     \apptocdots{black!25}\makebox[1.6em][r]{\textcolor{lucid}{\etocpage}}\par}{}%
  \etocsetstyle{subsection}{}{\par\addvspace{2.2pt}}%
    {\noindent\hangindent=4.3em\hangafter=1\hspace*{1.8em}\makebox[2.5em][l]{\textcolor{tocgray}{\etocthenumber}}\etoclink{\etocname}%
     \apptocdots{black!18}\makebox[1.6em][r]{\textcolor{tocgray}{\etocpage}}\par}{}%
  \tableofcontents
  \endgroup\clearpage}

\newtcolorbox{exbox}{appstatement,colback=defgreenlight,colframe=defgreen!50,before skip=12pt plus 2pt,after skip=9pt plus 2pt}
\newcommand{\lu}[1]{{\setlength{\fboxsep}{1pt}\colorbox{lucidlight}{\textcolor{lucid}{$#1$}}}}%
\newcommand{\lucidline}[1]{{\setlength{\fboxsep}{2pt}\colorbox{lucidlight}{$\displaystyle #1$}}}%
\newcommand{\rowlab}[2]{\text{\normalsize\textcolor{#1}{\textsc{#2}}}}

\title{Lucid Dreaming for World Models: Learning to Doubt Imagination and Decide by Trust}

\author{%
Ziqi Wen$^{1,2}$ \quad Ting Xu$^{1,2}$ \quad Lianyu Wang$^{1,2}$ \quad Xian Lin$^{1,2}$ \quad Yanda Meng$^{3}$\\
\bfseries Huazhu Fu$^{4}$ \quad Meng Wang$^{1,2,\text{\Letter}}$ \quad Ching-Yu Cheng$^{1,2,5,6}$\\[6pt]
\normalfont\fontsize{6.71}{8}\selectfont $^{1}$Centre for Innovation and Precision Eye Health, Yong Loo Lin School of Medicine, National University of Singapore, Singapore 119228, Singapore\\
\normalfont\fontsize{6.71}{8}\selectfont $^{2}$Department of Ophthalmology, Yong Loo Lin School of Medicine, National University of Singapore, Singapore 119228, Singapore\\
\normalfont\fontsize{6.71}{8}\selectfont $^{3}$Bioengineering Program, Biomedical Sciences Division (BioMed), King Abdullah University of Science and Technology (KAUST), Thuwal, Saudi Arabia\\
\normalfont\fontsize{6.71}{8}\selectfont $^{4}$Institute of Advanced Intelligence and Computing (IAIC), Agency for Science, Technology and Research (A*STAR), Singapore 138632, Singapore\\
\normalfont\fontsize{6.71}{8}\selectfont $^{5}$Singapore Eye Research Institute, Singapore National Eye Centre, Singapore 169856, Singapore\\
\normalfont\fontsize{6.71}{8}\selectfont $^{6}$Ophthalmology \& Visual Sciences Academic Clinical Program (EYE ACP), Duke-NUS Medical School, Singapore 169856, Singapore
}

\AtBeginDocument{\newgeometry{textheight=9in,textwidth=6.00in,top=1in,headheight=14pt,headsep=25pt,footskip=30pt}}
\begin{document}
\etocdepthtag.toc{main}

\maketitle
{\renewcommand{\thefootnote}{}\footnotetext{\Letter~Corresponding author.}}

\begin{figure}[H] 
\centering
  \includegraphics[width=0.9\linewidth]{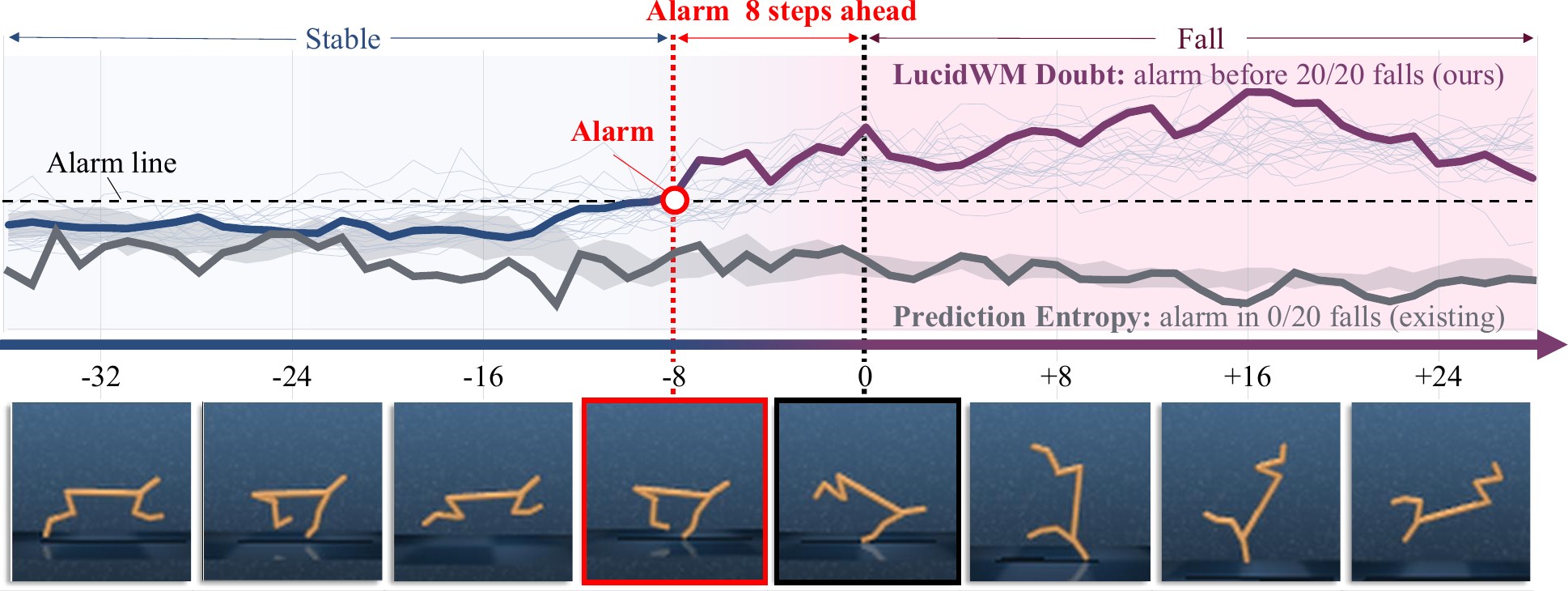}
  \captionsetup{width=0.9\linewidth}
  \captionof{figure}{\textbf{Our world model with doubt.} The doubt is read from the same head that generates its imagination, with no added parameter. The dark line takes one real fall as an example: the model reports its doubt eight steps before the fall and stays above its alarm line throughout it. We show that it reports all twenty consecutive real falls in the same way, every one before the fall begins. The usual readout, the entropy of the prediction, reports none of them.}
  \label{fig:teaser}
\end{figure}
\vspace{2pt}
\begin{abstract}
World models enable agents to learn and plan in imagination, but predictions beyond their experience can become unreliable and mislead decisions. Existing uncertainty estimates derived from predictions can remain overconfident on unfamiliar state-action pairs. We propose the Lucid World Model (LucidWM), which learns doubt from experience and propagates trust through imagination. By integrating Subjective Logic into categorical latent transitions, LucidWM distinguishes predicted outcomes from their evidential support and assigns each transition a degree of doubt. The complement of this doubt defines transition-level trust, which accumulates multiplicatively along imagined trajectories to reweight returns for policy learning and guide action selection. Uncertainty estimation requires no additional parameters or forward passes. Evaluated on four base world models against seventeen uncertainty readouts, LucidWM detects environmental changes and signals uncertainty during action-corrupted rollouts. 
In a controlled navigation case study, acting on trust reduces the number of steps required to reach the goal from 362 to 190. Fifteen demonstration videos show how LucidWM doubts its dreams and acts on that doubt. Videos are available at \mbox{\url{https://lucidwm.github.io}}.
\end{abstract}

\section{Introduction}

A world model is learned from experience and predicts how the world responds to an agent's actions. With it, the agent can dream: it feeds each predicted state back into the model and imagines future trajectories without further interaction with the environment \citep{hafner2025dreamerv3}. These trajectories support behaviour learning and planning.
Such imagination is especially valuable when real trials are costly \citep{wu2022daydreamer}, unsafe \citep{gao2024vista} or impossible to repeat under identical conditions \citep{xu2026meddreamer}. Yet it is often needed to evaluate actions and states beyond the agent's experience, precisely where its world model is least reliable \citep{yu2020mopo,kidambi2020morel}. Without sufficient experiential support, the model must extrapolate, and prediction errors can compound as each imagined state becomes the input to the next transition \citep{janner2019mbpo,asadi2018lipschitz}.

An unreliable dream is dangerous when believed. The agent may learn behaviours that exploit model errors and plan towards states that are implausible in the real environment \citep{kurutach2018metrpo}. Existing methods use uncertainty estimates to propagate predictive uncertainty or penalise unreliable model-generated transitions \citep{chua2018pets,yu2020mopo}.
Common signals include the entropy or variance of a predictive distribution and disagreement among ensemble members \citep{sekar2020plan2explore}.
However, a sharp prediction or agreement among models does not, by itself, establish that a queried transition is supported by experience. In our evaluations, these signals can remain insensitive to unfamiliar inputs or even become more confident as prediction reliability deteriorates (Fig.~\ref{fig:teaser}, Table~\ref{tab:main}).

\def\WFadj{0.10}
\begin{wrapfigure}{r}{0.552\linewidth}
  \vspace{-13.5pt}
  \centering
  \includegraphics[width=\linewidth]{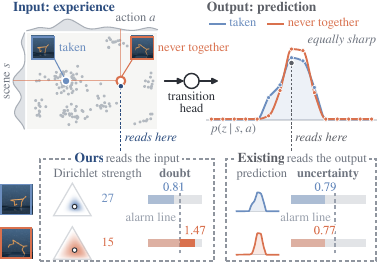}
  \caption{\textbf{Same prediction, different experience.} Top: a scene and an action never experienced together are predicted as sharply as a pair that was (real frames: at a fall's onset, and 30 steps before it). Bottom: ours reads the Dirichlet behind the input, whose strength and doubt separate the two; the entropy of the output cannot. Bars: each readout over its own alarm line.}
  \label{fig:story}
  \vspace{-8pt}
\end{wrapfigure}
We trace these failures to two distinct limitations. \textbf{(i) Doubt concerns support for the input, not just the shape of the output.} A state and an action may each be familiar while their combination is unsupported by observed transitions. The model can still produce a sharp prediction \citep{hein2019relu}, so a readout of predictive spread may report confidence despite limited evidence for that transition (Fig.~\ref{fig:story}). \textbf{(ii) Trust depends on the rollout, not just the current step.} Each imagined state inherits uncertainty from the transitions that produced it. A prediction should therefore be trusted only to the extent that its preceding trajectory is supported, back to the observed starting state. A pointwise uncertainty score alone does not capture this dependence \citep{berger2026biased}.

We propose to learn uncertainty from the evidence that experience provides for a transition, rather than infer it solely from the resulting prediction. The world model can then dream lucidly: just as a lucid dreamer knows a dream from waking life, it distinguishes what it predicts from the evidence supporting that prediction. We use the existing transition-head logits to parameterise a Dirichlet distribution over categorical next-state probabilities. Its mean gives the prediction, its concentration above the prior represents learned evidential strength, and the prior's share of the total concentration defines the step's doubt \citep{josang2016subjective}. Our training objective ties evidence to observed transitions and discourages unsupported evidence. During imagination, we accumulate the complements of these doubts multiplicatively into trajectory trust. This trust governs how strongly the agent relies on imagined continuations when learning and deciding.

\begin{figure}[t]
  \centering
  \includegraphics[width=\linewidth]{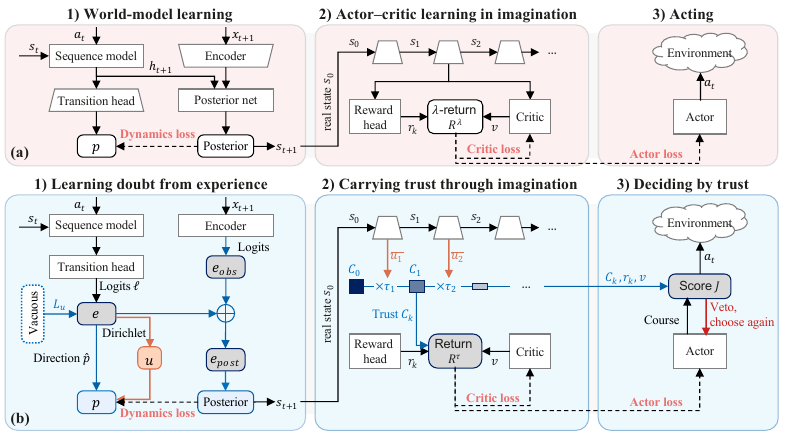}
  \caption{\textbf{Doubt is learned from experience and carried as trust.} (a)~A standard world model; (b)~LucidWM, changes in grey and blue, doubt in orange. 1)~The head's evidence $e$ yields a doubt $u$ in place of the fixed $\varepsilon$; the discipline $\mathcal L_u$ lets evidence recede toward the vacuous opinion where the dynamics loss does not hold it up. 2)~Doubt sets each imagined step's trust $\tau_k$, compounded into $C_k$; $R^{\tau}$ keeps $\lambda\tau_k$ of each step where $R^{\lambda}$ keeps $\lambda$. 3)~A course whose trust collapses is vetoed.}
  \label{fig:method}
\end{figure}
We realise this as the Lucid World Model (LucidWM), which integrates doubt and trust into three stages (Fig.~\ref{fig:method}). \textbf{(1) Learning doubt from experience.} We formulate the categorical transition head as a Subjective Logic (SL) opinion, with evidence measured above the Dirichlet prior, and update the posterior by fusing evidence from observation. \textbf{(2) Carrying trust through imagination.} We incorporate trust into the standard $\lambda$-return, reducing reliance on uncertain imagined continuations. Setting trust to one recovers the standard return. \textbf{(3) Deciding by trust.} Before acting, we evaluate the doubt associated with candidate actions and score their imagined futures using trust. LucidWM is designed for world models with categorical latent states and requires no additional learnable parameters or forward passes for uncertainty estimation.

Our contributions are threefold. \textbf{(i)} We identify how common predictive readouts can miss unsupported transitions and fail to capture cumulative uncertainty along imagined trajectories. \textbf{(ii)} We propose LucidWM, integrating evidential doubt and cumulative trust into world model learning, imagination and decision. \textbf{(iii)} Across four world model backbones and seventeen uncertainty readouts, our results demonstrate improved detection of environment changes after observation and rollout drift before decision making. In the reported fall and goal-reaching experiments, LucidWM raises an alarm before each of 20 evaluated falls and reaches the goal in 190 steps, compared with 362 for the base model. 

\section{Preliminaries and Related Work}\label{sec:prelim}

\textbf{World models.}
A world model \citep{ha2018worldmodels,hafner2019planet} has two parts, an encoder that compresses each observation $x_t$ into a latent state $s_t$, and a transition that predicts the next state from $s_t$ and the action $a_t$. Most recent models make the latent discrete, whether the transition is recurrent or a transformer \citep{hafner2025dreamerv3,morihira2026rdreamer,burchi2025emerald,zhang2026ocstorm}. Its stochastic part is $G$ categorical variables of $K$ classes, and the transition predicts each, $z_{t+1}$, with a categorical head of $K$ logits $\ell(s_t,a_t)$, a softmax with a small uniform mass mixed in so that no class is ever ruled out \citep{hafner2025dreamerv3}:
\begin{equation}
  p(z_{t+1}\mid s_t,a_t)\;=\;(1-\varepsilon)\,\mathrm{softmax}\big(\ell(s_t,a_t)\big)\;+\;\varepsilon/K,
  \label{eq:unimix}
\end{equation}
where $\varepsilon$ is a fixed constant, the same at every state. In training, a second head, the posterior $q(z_{t+1}\mid s_t,a_t,x_{t+1})$, reads the observation that actually arrived, and the transition head is trained to match it. Imagination samples from $p$ for $H$ steps, feeding the sampled state back as the next input, with a reward head supplying $r_t$; an actor $\pi$ and a critic $v$, which estimates the value of a state, are then trained on the imagined trajectory with the $\lambda$-return \citep{hafner2025dreamerv3,micheli2023iris}:
\begin{equation}
  R^{\lambda}_t\;=\;r_t+\gamma\big[(1-\lambda)\,v(s_{t+1})+\lambda\,R^{\lambda}_{t+1}\big],\qquad R^{\lambda}_{H}=v(s_{H}),
  \label{eq:lambda}
\end{equation}
where $\gamma$ is the discount and the fixed $\lambda$ sets, at every step, how far the return continues into the imagined future rather than falling back on the critic (we omit DreamerV3's continuation flag).

\textbf{Subjective Logic.}
SL \citep{josang2016subjective} extends probability by recording how much evidence a belief rests on, so that not knowing is told apart from knowing that an outcome is unlikely. A belief over $K$ outcomes, an \emph{opinion}, is a Dirichlet over the outcome probabilities whose concentration adds the evidence $e_i\ge0$ observed for each outcome $i$ to a prior of weight $W$ spread by a base rate, uniform throughout this paper. With $S=\sum_i e_i$ the total evidence and $\hat p=e/S$ its direction, the concentration, the uncertainty mass and the mean are:
\begin{equation}
  \alpha_i=e_i+W/K,\qquad u=W/(W+S),\qquad P_i=\alpha_i/(W+S)=(1-u)\,\hat p_i+u/K.
  \label{eq:opinion}
\end{equation}
The mean $P$ is the probability the opinion offers to a decision, and the uncertainty mass $u$, the prior's share of the Dirichlet's strength $W+S$, is one before any evidence and falls as evidence accumulates. We observe that $P$ has exactly the form of Eq.~\eqref{eq:unimix}, with the fixed $\varepsilon$ replaced by $u$.

\textbf{Uncertainty in world models.}
Current world models draw their uncertainty from three sources. The first is the prediction of a single model, read as the entropy or variance of the predicted distribution, its top probability, or the likelihood it gave to the observation that then arrived \citep{malik2019calibrated,savov2025autoexplore}. The second is a head fitted on the model's frozen features, scoring a state by the distance of its features from those seen in training \citep{lee2018mahalanobis} or by how poorly a trained network predicts a fixed random network's output on them \citep{burda2019rnd}; it sees the state, not the action. The third is disagreement among several models trained on the same data \citep{lakshminarayanan2017ensembles}, used as an exploration reward in Plan2Explore \citep{sekar2020plan2explore}, as a signal to halt imagination in MOReL \citep{kidambi2020morel}, and to weight, mask or truncate imagined rollouts \citep{buckman2018steve,pan2020m2ac,frauenknecht2024macura}. All three are read off what the model produces. We read uncertainty from the evidence behind each imagined step instead. An evidential head reads such evidence for a single prediction \citep{sensoy2018edl,amini2020der}; a world model's prediction is its next input, so its doubt must be carried along the rollout, fused with the observation that follows, and allowed to change what the agent learns.

\section{Method}\label{sec:method}
A standard world model pins two constants: its head, Eq.~\eqref{eq:unimix}, mixes the same $\varepsilon$ into every prediction, and its return, Eq.~\eqref{eq:lambda}, keeps the same share $\lambda$ of every imagined step. We unfold both into quantities read from the head's own logits, the \emph{doubt} of a transition and the \emph{trust} in a step, and follow them through learning (Sec.~\ref{sec:doubt}), imagination (Sec.~\ref{sec:trust}) and decision (Sec.~\ref{sec:decide}), as Fig.~\ref{fig:method} summarises.

\subsection{Learning doubt from experience}\label{sec:doubt}
\paragraph{Reading doubt from the transition head.}
The transition head already computes what we need: its logits, read as the evidence of an opinion. For each categorical variable of the latent, the head's $K$ logits are read as $e=\mathrm{softplus}\,\ell(s_t,a_t)$, a nonnegative amount of evidence per class. By Eq.~\eqref{eq:opinion}, their total $S$ is the experience the head brings to the scene and action it is asked about, $\hat p=e/S$ is its direction, and $u=\max\{W/(W+S),\,\varepsilon\}$ is the \emph{doubt} of the transition, the share of the prediction that no experience supports, floored at the standard head's $\varepsilon$ so that the lucid head never trusts a transition more than the standard head does. The prediction is the direction of the evidence; the doubt is its amount. Our head spreads that share uniformly over the classes:
\begin{align}
  \rowlab{black!55}{standard head}\;\; & p(z_{t+1}\mid s_t,a_t)=(1-\varepsilon)\,\mathrm{softmax}\big(\ell(s_t,a_t)\big)+\varepsilon/K, \tag{\ref{eq:unimix}}\\
  \rowlab{lucid}{our lucid head}\;\; & \lucidline{p(z_{t+1}\mid s_t,a_t)=(1-\lu{u(s_t,a_t)})\,\hat p(z_{t+1}\mid s_t,a_t)+\lu{u(s_t,a_t)}/K.} \label{eq:head}
\end{align}
\begin{keybox}
\begin{proposition}[Standard heads are constant-doubt opinions]\label{thm:two}
Any head of the form Eq.~\eqref{eq:unimix} is an instance of Eq.~\eqref{eq:head} with direction $\mathrm{softmax}(\ell)$ and doubt pinned at $u\equiv\varepsilon$; equivalently, through Eq.~\eqref{eq:opinion}, it asserts the same evidence total $S\equiv W(1-\varepsilon)/\varepsilon$ for every transition. For example, with DreamerV3's $\varepsilon=0.01$ and $W=2$, $S=198$ (App.~\ref{app:dirichlet}).
\end{proposition}
\end{keybox}
A standard world model thus assigns every transition, including those it has never experienced, the same amount of evidence. Ours lets the count vary with the transition: prediction and doubt become separate quantities, $\hat p$ can be peaked while $S$ is small, and only $S$ sets $u$. The head predicts as before and reports its doubt at every imagined step from the same forward pass. Read alone, $S$ is only a name for a total; what makes it track experience is how the head is trained, to which we turn next.

\paragraph{Binding evidence to experience.}
Evidence should grow with experience and with nothing else. The dynamics loss, which trains the transition head toward the posterior, does not enforce this: it decides how the evidence for a transition is allocated across classes, and has zero gradient in $u$ over a whole interval of doubts (App.~\ref{app:blind}), so the total evidence is free. We close the gap with a regulariser, the \emph{evidence discipline}: a pull toward the vacuous opinion, the opinion with no evidence and $u=1$. Evidential classifiers use such a pull to shed evidence the labels do not support \citep{sensoy2018edl}; here its job is different: with the fit blind to the total, the discipline alone sets it:
\begin{equation}
  \mathcal L_u=\beta_u\sum_{g}\mathrm{KL}\Big(\mathrm{Dir}\big(e^{g}(s_t,a_t)+(W/K)\,\mathbf 1\big)\,\Big\|\,\mathrm{Dir}\big((W/K)\,\mathbf 1\big)\Big),
  \label{eq:discipline}
\end{equation}
where $g$ runs over the $G$ categorical variables and $\beta_u$ weights the term. The dynamics loss now holds evidence up only on the scene--action pairs the agent has taken, and the discipline lets it recede wherever nothing holds it up. On the training distribution the head settles at the most doubtful opinion consistent with the fit (App.~\ref{app:track}); beyond it the doubt rises, as Sec.~\ref{sec:experience} shows.
\paragraph{Fusing observation into prediction.}
What the model observes should reduce its doubt and never raise it. In a standard world model the posterior is a separate network that reads $(s_t,a_t,x_{t+1})$ and replaces the prediction outright: what the prediction knew is discarded, and what the observation adds is never counted. Our posterior instead fuses the observation into the prediction with a fusion operator $\oplus$ of SL \citep{han2023tpami,xu2024rcml}; we use cumulative fusion, which adds evidence \citep{josang2016subjective} (App.~\ref{app:fusion}). A linear layer, which replaces the posterior network, reads $x_{t+1}$ alone and reports observation evidence $e_{\mathrm{obs}}(x_{t+1})$; this keeps the two sources independent, so no evidence is counted twice:
\begin{equation}
  e_{\mathrm{post}}(s_t,a_t,x_{t+1})=e(s_t,a_t)\oplus e_{\mathrm{obs}}(x_{t+1}).
  \label{eq:fusion}
\end{equation}
The posterior is the mean of the fused opinion, Eq.~\eqref{eq:opinion} with $e_{\mathrm{post}}$, from which $z_{t+1}$ is drawn in place of $q$. Fusion never raises the doubt, and brings two more properties: a vacuous observation leaves the state unchanged, and an informative one moves the state toward the observation by exactly the share of doubt it removes. It is also where experience enters the head: the dynamics loss pulls the prediction toward this posterior, whose sharpness bounds the doubt the prediction may keep (App.~\ref{app:interval}), so evidence accrues where observations have sharpened the posterior.
\subsection{Carrying trust through imagination}\label{sec:trust}
\paragraph{Carrying doubt forward as trust.}
The trust in an imagined step is inherited: each imagined state is the model's own prediction, so a step can be trusted no further than its input. One might carry it in the state, by deducing each imagined opinion from the doubtful one before it; but that recursion settles after one step at a level set by the current step alone, and the lineage is lost (App.~\ref{app:lineage}). We carry it in the return instead, and show that it then compounds as the trust discounting of SL (App.~\ref{app:discount}). The standard return already holds a trust of a kind: at every step it keeps a fixed share $\lambda$ of the imagined future and hands the rest to the critic. We let that share follow the doubt. Let $\bar u_{t+1}$ be the doubt of the transition $(s_t,a_t)$, averaged over the $G$ variables, and $\tau_{t+1}=(1-\bar u_{t+1})/(1-\varepsilon)$ its trust, equal to one when the doubt sits on its floor $\varepsilon$; the share kept beyond step $t$ becomes $\lambda\tau_{t+1}$:
\begin{align}
  \rowlab{black!55}{standard return}\;\; & R^{\lambda}_t=r_t+\gamma\big[(1-\lambda)\,v(s_{t+1})+\lambda\,R^{\lambda}_{t+1}\big], \tag{\ref{eq:lambda}}\\
  \rowlab{lucid}{our lucid return}\;\; & \lucidline{R^{\tau}_t=r_t+\gamma\big[(1-\lambda\lu{\tau_{t+1}})\,v(s_{t+1})+\lambda\lu{\tau_{t+1}}\,R^{\tau}_{t+1}\big],} \label{eq:return}
\end{align}
where $R^{\tau}_t$ is our trust-weighted return, with the same terminal value $R^{\tau}_H=v(s_H)$. A step with high doubt passes less of the imagined future to the return and more to the critic's value at that step.
\begin{keybox}
\begin{corollary}\label{cor:return}
The $\lambda$-return Eq.~\eqref{eq:lambda} is Eq.~\eqref{eq:return} with constant trust, $\tau\equiv1$. With Prop.~\ref{thm:two}, a standard world model is our lucid world model with both constants pinned.
\end{corollary}
\end{keybox}

\paragraph{Compounding trust along the rollout.}
Trust compounds along the rollout. For a rollout of $H$ steps from a real state $s_0$, write $C_n=\prod_{k=1}^{n}\tau_{k}$, with $C_0=1$, for the trust carried to its $n$-th imagined step. Unrolled from $s_0$, the return shows what this product does:
\begin{equation}
  R^{\tau}_0=\sum_{n=1}^{H-1}\big(\lambda^{n-1}C_{n-1}-\lambda^{n}C_n\big)\,R^{(n)}\;+\;\lambda^{H-1}C_{H-1}\,R^{(H)},
  \label{eq:unrolled}
\end{equation}
where $R^{(n)}=\sum_{k=0}^{n-1}\gamma^{k}r_k+\gamma^{n}v(s_n)$ is the $n$-step return. The weights are nonnegative and sum to one, and the weight of the $n$-step return is the share lost between steps $n-1$ and $n$: one doubtful step lowers the weight of every return beyond it. With $\tau\equiv1$ the weights are the geometric weights $(1-\lambda)\lambda^{n-1}$ of Eq.~\eqref{eq:lambda} (App.~\ref{app:guarantees}). How far imagination counts is thus set by trust, not by a constant: the return reaches as deep as the model's experience reaches, and falls back on the critic beyond it.
\subsection{Deciding by trust}\label{sec:decide}
\paragraph{Deciding what to learn from.}
We train the actor and the critic on $R^{\tau}$ instead of $R^{\lambda}$, so trust decides how much of each imagined step the agent learns from. For this to be safe, reweighting must not change what the critic converges to, and it should make model errors hurt less. Both hold:
\par\penalty-200
\begin{keybox}
\begin{theorem}[Contraction and bias bound]\label{thm:guar}
For any fixed doubt profile: (i) the critic's update under Eq.~\eqref{eq:return} is a $\gamma$-contraction to $\hat V^{\pi}$, the value of $\pi$ in the imagined model, as under Eq.~\eqref{eq:lambda}; (ii) with the critic at the true value $V^{\pi}$ and $\delta_j$ the model's one-step error at imagined depth $j$ (App.~\ref{app:guarantees}):
\begin{equation}
  \big\lVert\,\mathbb E[R^{\tau}_0]-V^{\pi}\big\rVert_\infty\;\le\;\sum_{j=0}^{H-1}\gamma^{\,j}\lambda^{j}\,C_j\,\delta_j.
  \label{eq:bias}
\end{equation}
\end{theorem}
\end{keybox}
(i) holds because the weights of Eq.~\eqref{eq:unrolled} are set by the head, not by the value being learned: trust changes how fast the critic gets there, not where. (ii) weights each depth's model error by $\lambda^{j}C_j$, the share of the return that reaches that depth, so the bound tightens wherever doubt precedes error.
\paragraph{Deciding what to do.}
Doubt is a function of $(s_t,a)$, so it is known for any candidate action before the action is taken (App.~\ref{app:ordering}). We score a candidate future by the trust-weighted sum of its rewards:
\begin{equation}
  J(a_{0:H-1})=\sum_{k=0}^{H-1}\gamma^{k}\,C_k\,r_k+\gamma^{H}\,C_H\,v(s_H).
  \label{eq:gate}
\end{equation}
$J$ is Eq.~\eqref{eq:return} with $\lambda=1$ and the fallback set to zero: untrusted imagined reward contributes nothing, whereas in learning that share went to the critic, so a course the model cannot vouch for cannot win on the critic's word. The simplest use of $J$ is a veto: when trust in the agent's course collapses, every reward beyond that point drops out of $J$, so the agent withdraws from the course and chooses again (Sec.~\ref{sec:decision}). Together, Eqs.~\eqref{eq:head}, \eqref{eq:return} and \eqref{eq:gate} apply one rule three times: keep the share of a prediction that experience supports, and hand the rest to a fallback, the uniform distribution, the critic, or zero.

\section{Experiments}\label{sec:exp}
\begin{table}[t]
\centering
\caption{(A)~After observation: AUROC of the readout, frames of the opened room against familiar ones; $0.5$ is chance. (B)~Before decision: lift $\rho_1/\rho_0$, the readout under corrupted actions over the readout under true ones; $1$ is no response. Under each column its cost relative to the base ($+c$: heads with $c$ times its parameters; $m$~fwd: $m$ forward passes). Means over five seeds; bold is the best in a row; {\color{black!45}grey} numbers move the wrong way, falling as the model leaves its experience; a dash ({--}) is a readout the base cannot provide. Each test uses eleven of the seventeen readouts; reconstruction and one-step error need the arriving frame, so appear in (A) only.}
\label{tab:main}
\providecommand{\hd}[2]{\begin{tabular}[b]{@{}c@{}}{\footnotesize #1}\\[-1pt]{\scriptsize\color{black!60}#2}\end{tabular}}
\providecommand{\ww}[1]{{\color{black!45}#1}}
{\small\setlength{\tabcolsep}{2.18pt}\renewcommand{\arraystretch}{1.300}\setlength{\aboverulesep}{0pt}\setlength{\belowrulesep}{0pt}\setlength{\extrarowheight}{1.5pt}
\begin{tabular*}{\linewidth}{@{\extracolsep{\fill}}ll@{\hspace{6.5pt}}>{\columncolor{lucidlight}[3pt][3pt]}c@{\hspace{6.5pt}}*{11}{c}@{}}
\toprule
\multicolumn{4}{@{}l}{\textbf{(A) After observation}} & \multicolumn{4}{c}{\footnotesize free readouts} & \multicolumn{4}{c}{\footnotesize fitted heads} & \multicolumn{2}{c@{}}{\footnotesize multiple forwards} \\
\cmidrule(lr){5-8}\cmidrule(lr){9-12}\cmidrule(l){13-14}
{\footnotesize base} & {\footnotesize task} & \hd{\textbf{ours}}{$1\times$} & \hd{base}{$1\times$} & \hd{entropy}{$1\times$} & \hd{KL}{$1\times$} & \hd{recon.}{$1\times$} & \hd{1-step}{$1\times$} & \hd{RND}{$+.18$} & \hd{RND-e}{$+.38$} & \hd{evid.}{$+.12$} & \hd{latent}{$+.62$} & \hd{self}{3 fwd} & \hd{deep}{$3\times$} \\
\midrule
DreamerV3 & \multirow{4}{*}{maze} & \textbf{0.77} & \ww{0.24} & \ww{0.31} & \ww{0.24} & \ww{0.06} & \ww{0.10} & \ww{0.39} & \ww{0.27} & \ww{0.48} & \ww{0.33} & \ww{0.21} & \ww{0.21} \\
R2-Dreamer & & \textbf{0.83} & \ww{0.37} & \ww{0.36} & 0.52 & -- & 0.50 & \ww{0.45} & 0.66 & \ww{0.30} & \ww{0.33} & \ww{0.41} & -- \\
EMERALD & & \textbf{0.96} & 0.72 & \ww{0.40} & \ww{0.27} & \ww{0.05} & \ww{0.11} & \ww{0.22} & \ww{0.20} & 0.85 & 0.83 & \ww{0.26} & -- \\
OC-STORM & & \textbf{0.82} & 0.75 & 0.75 & \ww{0.37} & \ww{0.09} & \ww{0.16} & 0.53 & \ww{0.46} & 0.57 & 0.59 & \ww{0.24} & -- \\
\midrule
\multicolumn{4}{@{}l}{\textbf{(B) Before decision}} & \multicolumn{2}{c}{\footnotesize free readouts} & \multicolumn{5}{c}{\footnotesize fitted heads} & \multicolumn{3}{c@{}}{\footnotesize multiple forwards} \\
\cmidrule(lr){5-6}\cmidrule(lr){7-11}\cmidrule(l){12-14}
{\footnotesize base} & {\footnotesize task} & \hd{\textbf{ours}}{$1\times$} & \hd{base}{$1\times$} & \hd{entropy}{$1\times$} & \hd{max p.}{$1\times$} & \hd{Mahal.}{$+.04$} & \hd{kNN}{$+2.4$} & \hd{RND}{$+.13$} & \hd{evid.}{$+.10$} & \hd{latent}{$+.52$} & \hd{MC drop}{20 fwd} & \hd{Laplace}{20 fwd} & \hd{snap.}{4 fwd} \\
\midrule
DreamerV3 & walker & \textbf{1.30} & \ww{0.85} & 1.03 & 1.00 & 1.06 & 1.07 & \ww{0.80} & 1.00 & 1.04 & 1.02 & 1.01 & \ww{0.93} \\
 & cheetah & \textbf{2.33} & \ww{0.70} & 1.14 & 1.14 & 1.10 & \ww{0.95} & \ww{0.93} & 1.00 & 1.04 & 1.04 & \ww{0.88} & \ww{0.94} \\
 & pendulum & \textbf{1.33} & 1.00 & 1.03 & 1.04 & 1.08 & 1.02 & \ww{0.96} & 1.00 & \ww{0.91} & 1.07 & \ww{0.84} & \ww{0.95} \\
 & Crafter & \textbf{1.49} & \ww{0.84} & 1.02 & \ww{0.97} & \ww{0.88} & \ww{0.93} & 1.02 & 1.00 & 1.01 & \ww{0.99} & \ww{0.92} & \ww{0.95} \\
R2-Dreamer\rule{0pt}{15pt} & walker & \textbf{3.11} & 1.19 & 1.00 & 1.01 & 2.45 & 1.14 & 1.36 & \ww{0.99} & 1.33 & 1.31 & 1.37 & -- \\
 & cheetah & \textbf{2.76} & \ww{0.99} & 1.01 & 1.02 & 1.17 & 1.01 & 1.18 & 1.00 & 1.04 & 1.21 & 1.01 & -- \\
 & pendulum & \textbf{1.53} & 1.02 & \ww{0.88} & \ww{0.87} & 1.01 & \ww{0.94} & 1.11 & \ww{0.99} & \ww{0.96} & \ww{0.96} & 1.11 & -- \\
 & Crafter & \textbf{1.15} & \ww{0.81} & 1.02 & 1.00 & 1.05 & 1.02 & 1.05 & 1.00 & 1.04 & 1.00 & 1.14 & 1.06 \\
EMERALD\rule{0pt}{15pt} & walker & \textbf{3.13} & 1.10 & 1.55 & 1.67 & 2.08 & 1.42 & 1.51 & 1.00 & 1.53 & \ww{0.05} & \ww{0.06} & \ww{0.16} \\
 & cheetah & \textbf{3.67} & 1.04 & 1.42 & 1.57 & 2.58 & 1.56 & 1.42 & 1.00 & 1.72 & \ww{0.28} & \ww{0.29} & \ww{0.25} \\
 & pendulum & \textbf{2.06} & 1.03 & 1.02 & 1.03 & 1.39 & 1.80 & 1.87 & 1.00 & 1.34 & \ww{0.62} & \ww{0.66} & \ww{0.94} \\
 & Crafter & \textbf{1.34} & 1.01 & 1.16 & 1.16 & \ww{0.89} & \ww{0.89} & 1.01 & 1.00 & 1.01 & \ww{0.97} & \ww{0.97} & 1.00 \\
OC-STORM\rule{0pt}{15pt} & walker & \textbf{5.62} & 1.16 & 2.74 & 2.14 & 1.08 & \ww{0.93} & 2.48 & 1.00 & 4.18 & 4.37 & 1.08 & -- \\
 & cheetah & \textbf{4.30} & 1.06 & 2.21 & 2.08 & 1.12 & \ww{0.98} & 1.47 & 1.03 & 2.18 & 2.95 & \ww{0.69} & -- \\
 & pendulum & \textbf{2.29} & 1.03 & 1.27 & 1.14 & \ww{0.86} & \ww{0.89} & \ww{0.98} & 1.00 & \ww{0.86} & 1.01 & \ww{0.80} & -- \\
 & Crafter & \textbf{2.59} & 1.04 & 2.08 & 1.79 & \ww{0.82} & \ww{0.92} & \ww{0.85} & 1.01 & 1.35 & \ww{0.61} & \ww{0.55} & \ww{0.80} \\
\bottomrule
\end{tabular*}
}
\end{table}

\subsection{Experimental protocol}\label{sec:setup}
\textbf{Bases.}
Lucid dreaming applies to any world model with a categorical transition head, and we test it on four bases that span the current designs, in four independent code bases, with the same readout and the same protocol throughout. \textbf{DreamerV3} \citep{hafner2025dreamerv3} is a \emph{recurrent} state-space model with a \emph{pixel decoder}; \textbf{R2-Dreamer} \citep{morihira2026rdreamer} is recurrent but learns its representation \emph{without reconstruction}; \textbf{EMERALD} \citep{burchi2025emerald} predicts its latent tokens with a \emph{masked transformer}; \textbf{OC-STORM} \citep{zhang2026ocstorm} is a \emph{transformer} world model, used here in its visual-only branch. Between them they cover recurrent and transformer dynamics, reconstruction-based and reconstruction-free representations, and flat and spatial latent states.

\textbf{Environments.}
We use three environment families to produce situations the model has not experienced. In DeepMind Control (DMC) continuous control (\textbf{walker}, \textbf{cheetah}, \textbf{pendulum}, \textbf{finger}, \textbf{cartpole}), where the dynamics are smooth, we corrupt the actions fed to imagination; in a first-person \textbf{ViZDoom maze}, navigated from $64\times64$ pixels with discrete actions, we train on one map and then alter it, opening a sealed door or repainting the walls; in \textbf{Crafter}, a procedurally generated open world, we corrupt the actions too and let the agent explore and meet caves, lava and terrain it has never seen. Together they cover continuous and discrete actions, proprioceptive and pixel inputs, and situations we inject, build in, or let the agent meet on its own (App.~\ref{app:further}).

\textbf{Opponents.}
We compare against seventeen readouts, eleven in each of the two tests, in the direction their authors define. \textbf{Free readouts} of the same pass: entropy, maximum probability, the KL between posterior and prediction, reconstruction and one-step error, and our own formula on the unmodified base (\emph{base} in Table~\ref{tab:main}); \textbf{fitted heads} on frozen features: RND on the latent and on the encoder embedding, an evidential head, Mahalanobis and $k$-NN distances, a latent ensemble; \textbf{multiple forwards}: MC dropout, Laplace, snapshot, self- and deep ensembles. Before decision, every fitted head and ensemble receives the same categorical latent our readout uses, so that what is compared is the readout, not how much the head is allowed to see. Appendix~\ref{app:opponents} defines each readout and its cost.

\subsection{Main results}\label{sec:main}
\textbf{(A) After observation.}
\emph{This test asks whether the doubt read after an observation tells that the world has changed.} We open a sealed door in the maze and drive the same routes on the old and the new map; a readout is scored by its AUROC (Table~\ref{tab:main}). On every base the LucidWM readout recognises the new room. Most readouts of the unmodified base point the wrong way: seen through familiar walls, the new room yields a sharp prediction, so they read it as \emph{more} familiar than the rest of the maze, and the fitted heads and ensembles, at up to three times the cost, are no more reliable.

\textbf{(B) Before decision.}
\emph{This test asks whether the doubt read before a decision tells that an imagined future has left the model's experience.} Imagination starts from real states, and its actions are replaced by random ones with growing probability; we report the lift $\rho_1/\rho_0$, the readout under fully corrupted actions over that under the true ones, so $1$ is no response and a readout passes when it rises above it. The LucidWM readout rises and leads on every row, while the same formula on the unmodified base moves little or the wrong way: fed nonsense, a model without doubt barely notices (App.~\ref{app:ladder}).

\needfit{262pt}
\subsection{Doubt rises beyond experience}\label{sec:experience}
\def\WFadj{0.40}
\begin{wrapfigure}{r}{0.625\linewidth}
  \vspace{-14.0pt}
  \raggedright
  \includegraphics[width=\linewidth]{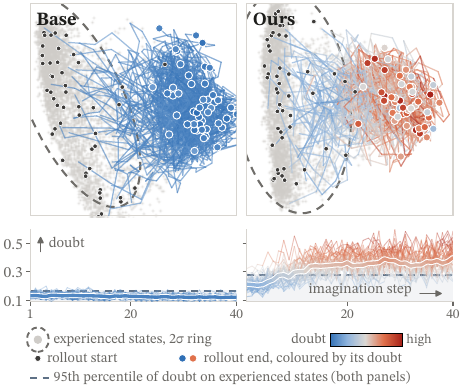}
  \caption{\textbf{Ours doubts the departure; base does not.} Top: experienced states and 48 random-action rollouts coloured by doubt. Bottom: doubt along the rollouts.}\label{fig:cloud}
  \vspace{-8pt}
\end{wrapfigure}
\textbf{Imagination leaving experience.}
\emph{This test asks whether doubt rises when imagination leaves what the model has experienced.} Fig.~\ref{fig:cloud} projects the visited states onto a plane, ringed at two standard deviations, and starts 48 random-action rollouts from real states in the cloud, forty steps each, colouring every step by doubt; below, each rollout's doubt and the median are drawn against the 95th percentile on experienced states. The rollouts of both models drift equally far out of the cloud, to twice the distance of any experienced state from its neighbours. Only LucidWM's doubt rises as they leave: within ten steps its median crosses the 95th percentile, and by the end nine in ten sit above it. The base's doubt never moves: its imagination has left everything it knows, and it is as confident as ever.

\textbf{Observation leaving experience.}
\emph{This test asks whether the doubt reads more than the surface of the frames.} We alter the maze in two ways (Fig.~\ref{fig:maze}): the walls of one room are repainted, which changes every pixel and nothing else; and a wall is opened onto a room that was sealed off during training, which changes the layout and shows the agent no texture it has not seen. For the new room we replay the same actions on the training map and read the difference. In the repainted room the lucid doubt rises far above its alarm line again and again; at the new room, it rises each time the agent turns to the opening and falls back as soon as it turns away. Under one alarm rule, set on familiar frames, the LucidWM readout raises an alarm in both rooms and the base's readout in neither. The doubt reads the structure beneath the pixels: every frame at the opening is made of familiar walls, and the doubt rises all the same.
\WFfinish{4}

\begin{figure}[H]
  \centering
  \includegraphics[width=0.938\linewidth]{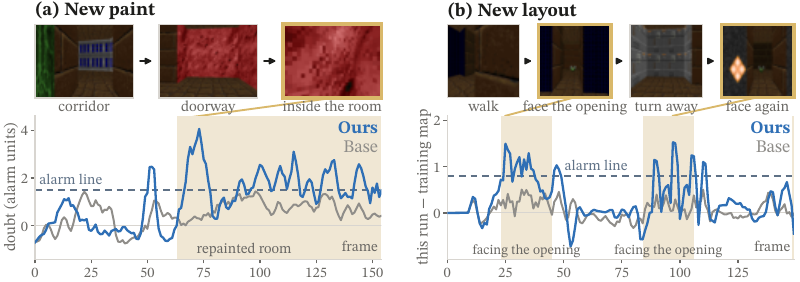}
  \caption{\textbf{The alarm fires where the world changed.} (a)~A drive into the repainted room. (b)~A patrol that twice faces the newly opened room. Doubt in each model's own units; dashed line: alarm line; sand: frames in the changed region, where the gold-outlined pictures were taken.}\label{fig:maze}
\end{figure}
\needfit{186pt}
\subsection{Doubt falls as experience grows}\label{sec:calibration}
\def\WFadj{1.35}
\begin{wrapfigure}{r}{0.632\linewidth}
  \vspace{-14.0pt}
  \raggedright
  \includegraphics[width=\linewidth]{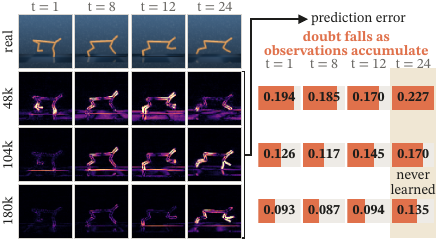}
  \caption{\textbf{Learning more, doubting less.} Rows: checkpoints of one run; columns: imagined steps. Tiles: pixel error, dark right, bright wrong. Shaded: chaotic horizon, never learned.}\label{fig:growth}
  \vspace{-8pt}
\end{wrapfigure}
\textbf{Experience accumulating.}
As the model sees more of the world, its doubt should drain, and drain where its predictions improve. We take one stretch of experience and, at six checkpoints of one run (three in Fig.~\ref{fig:growth}, all in App.~\ref{app:growth}), imagine the same future from it and compare it with what happened. At one step the doubt halves, falling at every checkpoint, and the error falls nearly 25-fold, though the two are measured independently: the error by comparing pixels, the doubt from the transition's evidence alone. Step 24 lies past the point where the dynamics turn chaotic, so no amount of training makes the prediction right; there the doubt stays the highest of the steps shown at every checkpoint.
\WFfinish{6}

\needfit{220pt}
\subsection{Ablating trust}\label{sec:imagination}
\def\WFadj{0.70}
\begin{wrapfigure}{r}{0.614\linewidth}
  \vspace{-14.0pt}
  \raggedright
  \includegraphics[width=\linewidth]{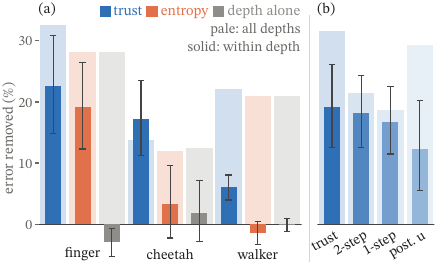}
  \caption{\textbf{Trust tells which imagined steps to keep.} Error removed by dropping the worst fifth of steps: (a)~three DMC tasks, (b)~cheetah from pixels. Pale: over all depths; solid: within each depth; whiskers: bootstrap 95\% intervals.}\label{fig:riskcov}
  \vspace{-8pt}
\end{wrapfigure}
\textbf{Trust as a filter.}
If trust means what it should, the imagined steps it trusts least should be the ones most wrong. We test this: from 480 real states per run we imagine 15 steps, rank every imagined step by a readout, drop the worst fifth, and measure how far the prediction error of the rest falls, on three DMC tasks. Deep steps are worse on average, so depth alone is a good filter, and over all depths it matches both trust and entropy (Fig.~\ref{fig:riskcov}a, pale bars). To see what a readout knows beyond depth, we rank within each depth (solid bars): depth alone drops to zero, entropy to almost nothing on two of three tasks, while trust keeps a gain of its own on all three, up to 23\%. Over all depths, compounded trust removes 32\% of the error against 19\% for the doubt of one step; within each depth, the posterior's doubt trails both (Fig.~\ref{fig:riskcov}b): compounding carries the depth of a dream, the doubt of each transition its drift, so trust reads both how far the dream has gone and how far it has drifted (App.~\ref{app:filmstrip}--\ref{app:window}).
\WFfinish{7}

\subsection{Acting on trust}\label{sec:decision}
\textbf{A repainted dead end.}
We run the same weights and seed twice through a dead end whose deep half was repainted after training: every wall is where it was, so a policy that learned the map walks in as before. Once the agent ignores its doubt (\emph{base}); once it acts on trust with the veto of Sec.~\ref{sec:decide} (\emph{ours}). The veto is one rule: when doubt holds above the alarm line for three frames, trust in the course collapses, the agent turns, and its policy resumes. For the first 24 steps the two are one agent. At step 24 the veto fires; ours turns, leaves the dead end and reaches the goal, an armour, at step 190, while base circles the room until step 262 and arrives at step 362 (Fig.~\ref{fig:decision}): the same policy, deciding 238 steps earlier and arriving in half the steps. Over its whole walk the entropy of base crosses its own line for a single frame, so the same rule would never have fired. On both altered maps, 143 of 145 vetoes fired inside a changed room, though the agent is never told where the changes are (App.~\ref{app:vetoes}).
\begin{figure}[H]
  \centering
  \includegraphics[width=0.600\linewidth]{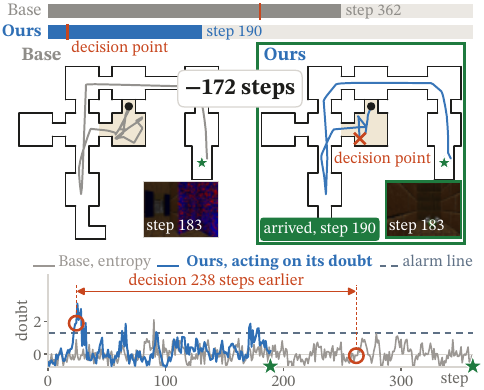}
  \caption{\textbf{One veto, half the steps.} Top: steps to the armour. Middle: both routes from the repainted dead end (sand) to the armour (star), and the two views at step 183, seven steps before ours arrives. Bottom: doubt along ours and the entropy of base, each on its own $p_{99}$ alarm line; circles: where ours turns and base leaves.}\label{fig:decision}
\end{figure}
\vspace{-15pt}
\FloatBarrier

\section{Conclusion}
In this work, we identified that existing uncertainty estimates for world models are read from the prediction, and therefore stay confident on exactly what the model has never experienced. To address this, we proposed the Lucid World Model (LucidWM), which learns uncertainty from experience: SL is built into the latent state transition, so that every predicted state carries a doubt, doubt compounds along imagination into trust, and trust weights both what the agent learns from it and what it decides to do. Experiments on four world-model bases and against seventeen readouts show that the doubt recognises a changed world after observation and a drifting rollout before decision, where prior readouts barely move or move the wrong way, and that acting on trust leads to safer choices and faster learning (App.~\ref{app:choose},~\ref{app:crafterlearn}). A standard world model is the special case with both constants pinned; unfolding them adds no parameter and no forward pass.

\setlength{\bibsep}{2pt plus 1pt}
\bibliography{references}
\bibliographystyle{lucid}

\clearpage
\appendix
\etocdepthtag.toc{app}
\appendixcontents
\appgroup{Overview}
\section{Notation}\label{app:notation}
\begin{table}[h]
  \caption{Symbols of Secs.~\ref{sec:prelim}--\ref{sec:exp}. Per variable: defined for each of the $G$ categorical variables, whose index $g$ is dropped when clear. The last column gives the value that recovers the standard world model.}
  \label{tab:notation}
  \centering\footnotesize\setlength{\tabcolsep}{4pt}\renewcommand{\arraystretch}{1.08}
  \begin{tabular}{@{}>{\raggedright\arraybackslash}p{2.68cm}>{\raggedright\arraybackslash}p{7.75cm}>{\raggedright\arraybackslash}p{2.12cm}>{\raggedright\arraybackslash}p{1.75cm}@{}}
    \toprule
    symbol & meaning & defined & standard \\
    \midrule
    \multicolumn{4}{@{}l}{\emph{World model and imagination}}\\
    $x_t,\ s_t,\ a_t$ & observation, latent state, action at time $t$ & Sec.~\ref{sec:prelim} & \\
    $G,\ K$ & number of categorical variables of the state; classes per variable & Sec.~\ref{sec:prelim} & \\
    $z_{t+1}$ & one categorical variable of the next state & Eq.~\eqref{eq:unimix} & \\
    $\ell(s_t,a_t)$ & the transition head's $K$ logits (per variable) & Eq.~\eqref{eq:unimix} & \\
    $p(z_{t+1}\mid s_t,a_t)$ & the head's prediction & Eqs.~\eqref{eq:unimix}, \eqref{eq:head} & \\
    $q(\cdot\mid s_t,a_t,x_{t+1})$ & posterior of a standard world model & Sec.~\ref{sec:prelim} & \\
    $\varepsilon$ & uniform share mixed into the head ($0.01$ in DreamerV3); floor of the doubt & Eq.~\eqref{eq:unimix} & \\
    $r_t,\ \gamma,\ H$ & imagined reward, discount, imagination horizon & Eq.~\eqref{eq:lambda} & \\
    $\pi,\ v$ & actor, critic & Eq.~\eqref{eq:lambda} & \\
    $\lambda$ & share of the imagined future kept at each step & Eq.~\eqref{eq:lambda} & \\
    $R^{\lambda}_t$ & $\lambda$-return & Eq.~\eqref{eq:lambda} & \\
    \midrule
    \multicolumn{4}{@{}l}{\emph{Opinion of the transition head (per variable)}}\\
    $e,\ e_i$ & evidence $\mathrm{softplus}\,\ell$; its entry for class $i$ & Eq.~\eqref{eq:opinion}, Sec.~\ref{sec:doubt} & \\
    $W$ & prior weight of the opinion & Eq.~\eqref{eq:opinion} & \\
    $S$ & total evidence $\sum_i e_i$ & Eq.~\eqref{eq:opinion} & $W(1-\varepsilon)/\varepsilon$ \\
    $\hat p$ & direction of the evidence, $e/S$ & Eq.~\eqref{eq:opinion} & $\mathrm{softmax}(\ell)$ \\
    $\alpha_i,\ P$ & Dirichlet concentration $e_i+W/K$; mean of the opinion & Eq.~\eqref{eq:opinion} & \\
    $u$ & doubt: uncertainty mass $W/(W+S)$, floored at $\varepsilon$ & Sec.~\ref{sec:doubt} & $\varepsilon$ \\
    $\mathcal L_u,\ \beta_u$ & evidence discipline and its weight & Eq.~\eqref{eq:discipline} & $\beta_u=0$ \\
    $e_{\mathrm{obs}}(x_{t+1})$ & evidence read from the observation alone & Eq.~\eqref{eq:fusion} & \\
    $\oplus,\ e_{\mathrm{post}}$ & fusion of opinions (here cumulative, which adds evidence); posterior evidence & Eq.~\eqref{eq:fusion} & \\
    \midrule
    \multicolumn{4}{@{}l}{\emph{Trust along imagination and decision}}\\
    $\bar u_{t+1}$ & doubt of the transition $(s_t,a_t)$, averaged over the $G$ variables & Sec.~\ref{sec:trust} & $\varepsilon$ \\
    $\tau_{t+1}$ & trust of that transition, $(1-\bar u_{t+1})/(1-\varepsilon)$ & Sec.~\ref{sec:trust} & $1$ \\
    $R^{\tau}_t$ & trust-weighted return & Eq.~\eqref{eq:return} & $R^{\lambda}_t$ \\
    $C_n$ & trust carried to imagined step $n$, $\prod_{k=1}^{n}\tau_k$, $C_0=1$ & Sec.~\ref{sec:trust} & $1$ \\
    $R^{(n)}$ & $n$-step return from the real state $s_0$ & Eq.~\eqref{eq:unrolled} & \\
    $\hat V^{\pi},\ V^{\pi}$ & value of $\pi$ in the imagined model; in the environment & Thm.~\ref{thm:guar} & \\
    $\delta_j$ & the model's one-step error at imagined depth $j$ & Thm.~\ref{thm:guar} & \\
    $J(a_{0:H-1})$ & trust-weighted score of a candidate course of actions & Eq.~\eqref{eq:gate} & $C\equiv1$ \\
    $\rho_1/\rho_0$ & lift: readout under corrupted actions over readout under true ones & Sec.~\ref{sec:main} & \\
    \bottomrule
  \end{tabular}
\end{table}
Symbols used inside a single proof in Apps.~\ref{app:count}--\ref{app:proofs} are defined where they appear.

\section{What LucidWM changes}\label{app:changes}
A standard categorical world model pins two constants: its head mixes the same $\varepsilon$ into every prediction, Eq.~\eqref{eq:unimix}, and its return keeps the same share $\lambda$ of every imagined step, Eq.~\eqref{eq:lambda}. LucidWM unfolds both into quantities read from the head's own logits and applies one rule at every stage: keep the share of a prediction that experience supports, and hand the rest to a fallback. Table~\ref{tab:changes} lists every change; the statements behind each row are Prop.~\ref{thm:two}, Prop.~\ref{prop:track}, Prop.~\ref{prop:seeing}, Thm.~\ref{thm:guar} and Prop.~\ref{prop:acting}. No parameter and no forward pass is added: the posterior is the fused evidence of Eq.~\eqref{eq:fusion} in place of a second network, and the readouts $u$, $\tau$ and $C$ are functions of quantities the model already computes.
\begin{table}[H]
  \caption{What LucidWM changes in a categorical world model. Every readout is a function of the head's logits; the standard model is recovered by $u\equiv\varepsilon$ and $\tau\equiv1$.}
  \label{tab:changes}
  \centering\footnotesize\setlength{\tabcolsep}{4pt}
  \begin{tabular}{@{}llll@{}}
    \toprule
    stage & standard & LucidWM & fallback \\
    \midrule
    prediction, Eq.~\eqref{eq:head} & $(1-\varepsilon)\,\mathrm{softmax}(\ell)+\varepsilon/K$ & $(1-u)\,\hat p+u/K$, $u$ from $S$ & uniform $1/K$ \\
    training, Eq.~\eqref{eq:discipline} & fit only & fit and evidence discipline & the vacuous opinion \\
    posterior, Eq.~\eqref{eq:fusion} & a second network $q$ & $e\oplus e_{\mathrm{obs}}$ & none: doubt never rises \\
    return, Eq.~\eqref{eq:return} & share $\lambda$ per step & share $\lambda\tau_{t+1}$, $\tau$ from $\bar u$ & critic $v$ \\
    decision, Eq.~\eqref{eq:gate} & reward counted in full & reward weighted by $C_k$ & zero \\
    \bottomrule
  \end{tabular}
\end{table}

\appgroup{Theory}
\section{Why the doubt reads experience}\label{app:count}
Section~\ref{sec:doubt} reads the total evidence behind a prediction as the experience the head brings to a scene and action. This appendix gives the grounds in five steps (Table~\ref{tab:appc}), with one example pair, in green boxes, followed through all of them. Every result treats the evidence of each scene--action pair on its own and its target as fixed; App.~\ref{app:track} ends with what carries over to a network with shared parameters.
\begin{table}[h]
  \caption{The argument of App.~\ref{app:count}, one result per step.}
  \label{tab:appc}
  \centering\small\renewcommand{\arraystretch}{1.12}
  \begin{tabular}{@{}l>{\raggedright\arraybackslash}p{0.75\linewidth}l@{}}
    \toprule
    step & result & where\\
    \midrule
    measure & the doubt $u=W/(W+S)$ sees the evidence only through its total $S$ & Eq.~\eqref{eq:opinion}\\
    ideal & in a table updated by Bayes' rule, $S$ is the joint count; an untaken pair has $u=1$ & Prop.~\ref{prop:count}\\
    obstacle & neither a readout of the prediction nor the fit can see $S$ & Prop.~\ref{prop:blind}\\
    remedy & the discipline keeps the most doubtful opinion with its mean; $u=1$ if untaken & Prop.~\ref{prop:track}\\
    payoff & the doubt sets the expected error of the prediction & Prop.~\ref{prop:error}\\
    \bottomrule
  \end{tabular}
\end{table}

\subsection{The doubt measures the evidence total}\label{app:dirichlet}
\paragraph{Notation.} An opinion over $K\ge2$ classes with evidence $e\ge0$ is the Dirichlet $\boldsymbol\omega\sim\mathrm{Dir}(\alpha)$ of Eq.~\eqref{eq:opinion}, with
\begin{gather*}
  S=\sum_ie_i,\qquad \alpha=e+\tfrac WK\mathbf 1,\qquad \alpha_0=W+S,\qquad u=\frac W{\alpha_0},\\
  \hat p=\frac eS,\qquad P=\mathbb E[\boldsymbol\omega]=\frac{\alpha}{\alpha_0}=(1-u)\,\hat p+\frac uK .
\end{gather*}
A \emph{target} $q$ is the distribution the fit pulls the prediction toward (in training, the posterior); the \emph{fit} is $\mathrm{KL}(q\,\Vert\,P)$; a \emph{readout} is any quantity computed from $P$. We say \emph{count} for the number $N$ of transitions taken from a pair and \emph{evidence total} for $S$.

The doubt is the share of the mean $P$ held by the prior: the uncertainty mass of SL \citep{josang2016subjective}, whose belief masses $e_i/\alpha_0$ make up the rest. It depends on the evidence only through $S$, is one at $S=0$ and falls strictly as $S$ grows. By conjugacy, observed counts add to $\alpha$, so evidence adds; this is the cumulative fusion of App.~\ref{app:fusion}.
\vspace{4pt}
\begin{proof}[Proof of Prop.~\ref{thm:two}]
A head of the form Eq.~\eqref{eq:unimix} is Eq.~\eqref{eq:head} with $\hat p=\mathrm{softmax}(\ell)$ and $u=\varepsilon$, that is, the opinion with the single total $S_\varepsilon=W(1-\varepsilon)/\varepsilon$ at every input. With the customary $W=2$ \citep{josang2016subjective} and DreamerV3's $\varepsilon=0.01$, $S_\varepsilon=198$.
\end{proof}
\begin{exbox}
\textbf{Example.} Take $K=3$ and $W=2$, and a pair taken $N=8$ times whose next latent fell in the three classes $6$, $2$ and $0$ times. Then $e=(6,2,0)$, $S=8$ and $u=0.2$: the prior holds a fifth of the prediction $P\approx(0.67,\,0.27,\,0.07)$.
\end{exbox}

\subsection{In a table, the doubt counts experience}\label{app:countprop}
\begin{keybox}
\begin{proposition}[Doubt counts joint experience]\label{prop:count}
Give each scene--action pair $(s,a)$ its own opinion $\boldsymbol\omega(s,a)\sim\mathrm{Dir}((W/K)\mathbf 1)$ over one categorical variable of the next latent, independent across pairs, and update it by Bayes' rule on the transitions experienced from it. After $N(s,a,z')$ transitions $(s,a)\to z'$ for each $z'$,
\begin{equation}
  e(s,a)=N(s,a,\cdot),\qquad u(s,a)=\frac{W}{W+N(s,a)},\qquad N(s,a)=\sum_{z'}N(s,a,z').
  \label{eq:count}
\end{equation}
The doubt depends on experience only through the joint count $N(s,a)$, and is one for a pair never taken, however often its scene and its action were seen apart.
\end{proposition}
\end{keybox}
\begin{proof}
The priors are independent across pairs and the likelihood $\prod_{(s,a)}\prod_{z'}\omega_{z'}(s,a)^{N(s,a,z')}$ factorises over pairs. By conjugacy the posterior of $(s,a)$ is $\mathrm{Dir}\big(N(s,a,\cdot)+(W/K)\mathbf 1\big)$, so $e(s,a)=N(s,a,\cdot)$, $S=N(s,a)$ and $u=W/(W+N(s,a))$; no count of another pair enters.
\end{proof}
\begin{exbox}
\textbf{Example.} The evidence $e=(6,2,0)$ is exactly the counts; a pair never taken has $e=0$ and $u=1$.
\end{exbox}
\paragraph{In words.} A familiar scene and a familiar action that were never taken together, the case of Fig.~\ref{fig:story}, therefore have doubt one. A network is not a table: it shares one function $e(s_t,a_t)$ across pairs and is trained by a fit, which, as the next step shows, cannot recover the count.

\subsection{The prediction and the fit cannot see the evidence total}\label{app:blind}\label{app:interval}
Many opinions share one prediction (Fig.~\ref{fig:family}a).
\begin{figure}[H]
  \centering
  \includegraphics[width=0.893\linewidth]{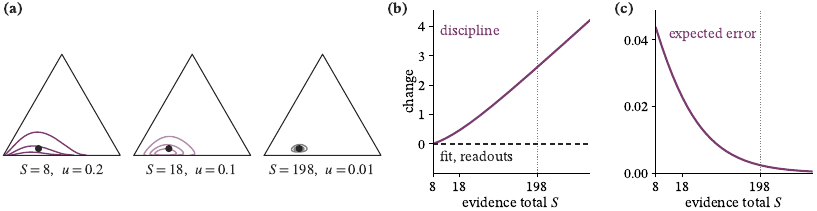}
  \caption{\textbf{One prediction, many evidence totals}, for the example pair. (a)~Dirichlets with the same mean $P$ (dot) and totals $S=8$, $18$ and $198$. (b)~Change from $S=8$ along this family: the fit and every readout stay constant (Prop.~\ref{prop:blind}), while the discipline grows with the total (Lemma~\ref{lem:disc}); among opinions with this mean it prefers the smallest total, $S=8$, the count. (c)~The expected error falls with the total (Prop.~\ref{prop:error}); at $S=198$, the total a standard head asserts, it is eighteen times smaller than at the count.}
  \label{fig:family}
\end{figure}
\begin{keybox}
\begin{proposition}[The prediction and the fit are blind to the total]\label{prop:blind}
Fix a prediction $P$ with full support and let $m=K\min_iP_i$.\\
(i) For every $u\in(0,m]$, the opinion $\mathrm{Dir}\big((W/u)P\big)$ has mean $P$ and doubt $u$, and it is the only one; no opinion with mean $P$ has doubt above $m$.\\
(ii) Every readout, and the fit to any fixed target $q$, take the same value on all of these opinions.
\end{proposition}
\end{keybox}
\begin{proof}
An opinion with doubt $u$ has $\alpha_0=W/u$, so its mean is $P$ exactly when $\alpha=(W/u)P$, that is, when
\begin{equation*}
  e=\frac Wu\,P-\frac WK\,\mathbf 1 ,
\end{equation*}
which is nonnegative exactly when $u\le m$. Readouts and the fit depend on the opinion only through $P$.
\end{proof}
\begin{exbox}
\textbf{Example.} Here $m=0.2$, and the opinions with $S=8$, $18$ and $198$ all have mean $P$ (Fig.~\ref{fig:family}a), so every readout and every fit takes one value on them (Fig.~\ref{fig:family}b).
\end{exbox}
\paragraph{In words.} The readouts of Table~\ref{tab:main} that are read off the prediction therefore cannot tell the unexperienced pair of Fig.~\ref{fig:story} from an experienced pair with the same sharp prediction. The fit cannot set the total either: its gradient in $u$ vanishes over the whole interval $(0,m]$. This is the non-identifiability of second-order learners whose loss sees only the predictive mean \citep[Thm.~3.2]{jurgens2024faithful}, while losses averaged over the second-order distribution collapse it to a point mass \citep[Thm.~1]{bengs2022pitfalls}. Something other than the fit must set the total.

\subsection{The discipline makes the doubt count experience}\label{app:track}
The evidence discipline sets the total. We model training in the form of Prop.~\ref{prop:count}: the discipline takes the place of the prior, charged once on every pair, and the fit the place of the likelihood, charged once per experienced transition. The discipline pulls every pair toward the vacuous opinion; Prior Networks pull toward a flat Dirichlet, and only on inputs out of distribution \citep{malinin2018prior}. A pair with $N$ experienced transitions and target $q$ is then trained by
\begin{align*}
  \mathcal L_N(e)&=N\,\mathrm{KL}\big(q\,\Vert\,P\big)+\beta_u\,D(\alpha),\\
  D(\alpha)&=\mathrm{KL}\big(\mathrm{Dir}(\alpha)\,\Vert\,\mathrm{Dir}(\tfrac WK\mathbf 1)\big)=\log\frac{\Gamma(\alpha_0)\,\Gamma(\frac WK)^K}{\Gamma(W)\prod_i\Gamma(\alpha_i)}+\sum_ie_i\big[\psi(\alpha_i)-\psi(\alpha_0)\big],
\end{align*}
with $\psi$ the digamma function. By Prop.~\ref{prop:blind} an opinion is fixed by its mean $P$ and its strength $\alpha_0=W+S\ge t_P:=W/(K\min_iP_i)$, and the fit depends on $P$ alone. The discipline grows with the total:
\begin{lemma}[The discipline grows with the evidence]\label{lem:disc}
(i) At a fixed mean $P$, $D$ is strictly increasing in $\alpha_0$ on $[t_P,\infty)$.\\
(ii) $D\ge0$, with equality only at $e=0$.
\end{lemma}
\begin{proof}
\textit{Idea.} At a fixed mean, the derivative of $D$ in the total is a positively weighted sum of $\xi(\alpha_i)-\xi(\alpha_0)$ for a decreasing $\xi$.

Differentiating gives $\partial D/\partial\alpha_i=e_i\,\psi'(\alpha_i)-S\,\psi'(\alpha_0)$. The function $\xi(y)=y\,\psi'(y)$ is strictly decreasing, since integrating $\psi'(y)=\int_0^\infty t\,e^{-yt}(1-e^{-t})^{-1}dt$ \citep[Eq.~6.4.1]{abramowitz1964handbook} by parts gives $\xi(y)=1+\int_0^\infty e^{-yt}h'(t)\,dt$ with $h(t)=t/(1-e^{-t})$ increasing, so also $\xi>1$; and every $\alpha_i<\alpha_0$. Hence
\begin{equation*}
  \frac{dD}{d\alpha_0}\Big|_{\alpha=\alpha_0P}=\frac1{\alpha_0}\sum_ie_i\big[\xi(\alpha_i)-\xi(\alpha_0)\big]>0
\end{equation*}
wherever $e\ne0$. (ii) is Gibbs' inequality.
\end{proof}
\begin{keybox}
\begin{proposition}[The discipline sets the evidence total]\label{prop:track}
Let $q$ have full support and $\beta_u>0$. $\mathcal L_N$ has a minimiser, and every minimiser has $u^\ast=K\min_iP^\ast_i$: it is the most doubtful opinion with its own mean. Moreover:\\
(i) without experience, $N=0$, the doubt is $u^\ast=1$;\\
(ii) with experience, $u^\ast\ge K\min_iq_i$; for $q$ not uniform, $u^\ast-K\min_iq_i=\Theta(\beta_u/N)$ as $\beta_u/N\to0$.
\end{proposition}
\end{keybox}
\begin{proof}
\textit{Idea.} The fit sees only the mean and the discipline grows with the total, so the minimiser takes the least total its mean allows; comparing it with the opinion that fits $q$ exactly then bounds the excess. Let $F(P)=\mathrm{KL}(q\,\Vert\,P)$, $D_{\min}(P)=D(t_PP)$ and $\kappa=\beta_u/N$.

\textit{Reduce to the mean.} Every opinion with mean $P$ has $\alpha_0\ge t_P$, so by Lemma~\ref{lem:disc}(i)
\begin{equation*}
  \mathcal L_N(e)\ \ge\ \Phi_N(P):=N\,F(P)+\beta_u\,D_{\min}(P),
\end{equation*}
with equality exactly at $\alpha_0=t_P$. A minimiser therefore has the smallest total its mean allows: its least likely class holds no evidence, and $u^\ast=K\min_iP^\ast_i$. $\Phi_N$ is continuous and, as $q$ has full support, grows without bound as $\min_iP_i\to0$, so a minimiser exists.

\textit{(i).} At $N=0$ the loss is $\beta_uD$, which vanishes only at $e=0$ (Lemma~\ref{lem:disc}(ii)).

\textit{(ii), upper bound.} The opinion with mean $q$ and doubt $K\min_iq_i$ (Prop.~\ref{prop:blind}) has loss $\beta_uD_{\min}(q)$, so $F(P^\ast)\le\kappa\,[D_{\min}(q)-D_{\min}(P^\ast)]\le\kappa\,D_{\min}(q)$, and $P^\ast\to q$ as $\kappa\to0$. Near $q$, $D_{\min}$ is Lipschitz and $F\ge\frac12\lVert P^\ast-q\rVert_1^2$ (Pinsker), so $\lVert P^\ast-q\rVert_1$ and hence $u^\ast-K\min_iq_i\le K\lVert P^\ast-q\rVert_1$ are $O(\kappa)$.

\textit{(ii), lower bound.} Let $N>0$ and $q$ be non-uniform (for uniform $q$, $e^\ast=0$ and $u^\ast=1=K\min_iq_i$); then $S^\ast>0$. At the class $j$ left without evidence, raising $e_j$ cannot lower the loss, the first inequality below; multiplying it by $\alpha^\ast_0/N$ gives the second:
\begin{equation*}
  N\Big(\frac1{\alpha^\ast_0}-\frac{Kq_j}{W}\Big)-\beta_uS^\ast\psi'(\alpha^\ast_0)\ \ge\ 0
  \quad\Longrightarrow\quad
  1-\frac{Kq_j}{u^\ast}\ \ge\ \kappa\,S^\ast\,\xi(\alpha^\ast_0)\ >\ \kappa\,S^\ast,
\end{equation*}
\looseness=-1 using $\xi>1$ (proof of Lemma~\ref{lem:disc}). Hence $u^\ast=KP^\ast_j>Kq_j/(1-\kappa S^\ast)\ge K\min_iq_i\,(1+\kappa S^\ast)$. By the upper bound $u^\ast\to K\min_iq_i$, so $K\min_iq_i\,S^\ast\to W(1-K\min_iq_i)>0$ and the excess is at least of order $\kappa$.
\end{proof}
\begin{exbox}
\textbf{Example.} With $\beta_u=0.1$ the discipline settles the pair at $u^\ast\approx0.24$ against the count's $0.2$, and returns $u^\ast=1$ for a pair never taken.
\end{exbox}
\paragraph{In words.} When the target is the posterior mean of Prop.~\ref{prop:count} after $N$ transitions, $K\min_iq_i\ge W/(W+N)$, the doubt of Eq.~\eqref{eq:count}, with equality when some class was never observed. By (ii), the discipline then never leaves an experienced pair more confident than the count, and errs toward doubt: by an excess of order $\beta_u/N$, and by more where every class was observed. Regularised evidential learners take, among equal fits, the one the regulariser prefers \citep[Sec.~3.4]{jurgens2024faithful}, and their uncertainty is then set by the regularisation weight \citep{bengs2022pitfalls,jurgens2024faithful}; here the weight sets only an excess of order $\beta_u/N$ above a level fixed by the target.

With a network, the evidence of every pair is one function $e(s_t,a_t)$ of shared parameters. Posterior Networks keep the count of Prop.~\ref{prop:count} in such a function with a learned density, setting the evidence of a class to its count times the density at the input \citep{charpentier2020posterior}; the lucid head needs no density model. The fit rewards a direction that generalises across pairs but never rewards evidence away from the transitions that demand it, so through the shared parameters the discipline lowers the total wherever no transition holds it up. Sec.~\ref{sec:exp} measures this on four bases: after observation and before decision (Table~\ref{tab:main}), along a rollout that leaves the experienced states (Fig.~\ref{fig:cloud}), and as experience accumulates (Fig.~\ref{fig:growth}).

\subsection{The doubt sets the expected error}\label{app:error}
\begin{keybox}
\begin{proposition}[Doubt sets the expected error of the prediction]\label{prop:error}
Let $\boldsymbol\omega\sim\mathrm{Dir}(\alpha)$ be the model's opinion on the distribution of a next latent variable, with mean $P$ and doubt $u$. Then
\begin{equation}
  \mathbb E\,\lVert\boldsymbol\omega-P\rVert_2^2=\big(1-\lVert P\rVert_2^2\big)\,\frac{u}{W+u},
  \label{eq:doubt-error}
\end{equation}
the share $u/(W+u)$ of the largest expected error, $1-\lVert P\rVert_2^2$, that any belief with mean $P$ can carry. For a fixed prediction it rises strictly with the doubt; a standard head gives every transition with that prediction the same share.
\end{proposition}
\end{keybox}
\begin{proof}
A Dirichlet has $\mathrm{Var}[\omega_i]=P_i(1-P_i)/(\alpha_0+1)$, so $\mathbb E\lVert\boldsymbol\omega-P\rVert_2^2=(1-\lVert P\rVert_2^2)/(\alpha_0+1)$, and $\alpha_0+1=(W+u)/u$. Any belief with mean $P$ has $\mathbb E\lVert\boldsymbol\omega-P\rVert_2^2=\mathbb E\lVert\boldsymbol\omega\rVert_2^2-\lVert P\rVert_2^2\le1-\lVert P\rVert_2^2$, with equality for the belief on the vertices. For fixed $P$, $1-\lVert P\rVert_2^2>0$ and $u/(W+u)$ rises with $u$; a standard head has $\alpha_0=W/\varepsilon$ everywhere.
\end{proof}
\begin{exbox}
\textbf{Example.} Here $1-\lVert P\rVert_2^2=0.48$: the expected error is $0.044$ at the count, $S=8$, but $0.0024$ at $S=198$, eighteen times smaller for the same prediction (Fig.~\ref{fig:family}c).
\end{exbox}
The error is taken under the model's own opinion; Fig.~\ref{fig:growth} tests whether that opinion is calibrated, and App.~\ref{app:tighter} says what calibration buys in the return.

\section{Why trust is carried in the return}\label{app:proofs}
\newtheorem*{thmrestated}{Theorem~\ref*{thm:guar}}
Section~\ref{sec:trust} carries the doubt of each imagined step as trust in the return, and Sec.~\ref{sec:decide} learns and decides by that trust. Carried through the imagined state instead, the doubt of a step would be lost one step later; carried in the return, it compounds, keeps the critic's fixed point and discounts every model error that follows a doubtful step. Table~\ref{tab:appd} lays out the argument; Apps.~\ref{app:fusion}--\ref{app:tighter} each give an example in a green box, and Apps.~\ref{app:lineage}--\ref{app:tighter} follow one imagined rollout.
\begin{table}[h]
  \caption{The argument of App.~\ref{app:proofs}, step by step.}
  \label{tab:appd}
  \centering\small\renewcommand{\arraystretch}{1.12}
  \begin{tabular}{@{}l>{\raggedright\arraybackslash}p{0.722\linewidth}l@{}}
    \toprule
    step & result & where\\
    \midrule
    start & at the real state, fusing in an observation never raises the doubt & Prop.~\ref{prop:seeing}\\
    alternative & carried in the state, the doubt of a step is lost one step later & Prop.~\ref{prop:lineage}\\
    choice & carried in the return, trust compounds, as SL trust discounting does & Prop.~\ref{prop:discount}\\
    safety & the lucid return keeps the critic's fixed point & Thm.~\ref{thm:guar}\\
    payoff & errors that follow a doubtful step are discounted, where it pays & Prop.~\ref{prop:critic}\\
    action & doubt ranks candidate actions before any is taken & Prop.~\ref{prop:acting}\\
    \bottomrule
  \end{tabular}
\end{table}

\paragraph{Notation.} Symbols are those of Table~\ref{tab:notation}. A rollout starts at a real state $s_0$ and takes $H$ imagined steps; step $k$ has the doubt $\bar u_k\ge\varepsilon$ of the transition $(s_{k-1},a_{k-1})$ and the trust $\tau_k\in[0,1]$, and $C_0=1$. We take $\lambda\in(0,1]$ and $\gamma\in(0,1)$, and write $M_n=\lambda^nC_n$ for the share of the return that reaches step $n$. A doubt profile is \emph{fixed} when $\bar u_1,\dots,\bar u_H$ are the same on every imagined rollout. Prop.~\ref{prop:seeing} uses the doubt before the floor, $u=W/(W+S)$, as in App.~\ref{app:count}.

\subsection{Observation never raises the doubt}\label{app:fusion}
\begin{keybox}
\begin{proposition}[Fusion never raises the doubt]\label{prop:seeing}
Fuse a prediction (evidence $e$, doubt $u$) with an observation (evidence $e_{\mathrm{obs}}$, total $S_{\mathrm{obs}}$, doubt $u_{\mathrm{obs}}=W/(W+S_{\mathrm{obs}})$) by cumulative fusion \citep[Sec.~12.3]{josang2016subjective}, $e_{\mathrm{post}}=e+e_{\mathrm{obs}}$, and let $u_{\mathrm{post}}$ and $P_{\mathrm{post}}$ be the doubt and mean of the result. Then\\
(i) $u_{\mathrm{post}}\le\min(u,u_{\mathrm{obs}})$, and $1/u_{\mathrm{post}}=1/u+S_{\mathrm{obs}}/W$;\\
(ii) $P_{\mathrm{post}}=(1-\theta)\,P+\theta\,\hat p_{\mathrm{obs}}$, with $\hat p_{\mathrm{obs}}=e_{\mathrm{obs}}/S_{\mathrm{obs}}$ and $\theta=1-u_{\mathrm{post}}/u$ the share of the doubt the observation removes; a vacuous observation, $S_{\mathrm{obs}}=0$, has $\theta=0$ and leaves the prediction unchanged.
\end{proposition}
\end{keybox}
\begin{proof}
The total $S+S_{\mathrm{obs}}$ is at least $S$ and at least $S_{\mathrm{obs}}$, and the doubt falls with the total; $W/u_{\mathrm{post}}=W+S+S_{\mathrm{obs}}=W/u+S_{\mathrm{obs}}$. For (ii), $\alpha_{\mathrm{post}}=\alpha+e_{\mathrm{obs}}$, so $P_{\mathrm{post}}=\big((W+S)\,P+S_{\mathrm{obs}}\,\hat p_{\mathrm{obs}}\big)/(W+S+S_{\mathrm{obs}})$, and $S_{\mathrm{obs}}/(W+S+S_{\mathrm{obs}})=1-u_{\mathrm{post}}/u$.
\end{proof}
\begin{exbox}
\textbf{Example.} Fuse the example pair of App.~\ref{app:count}, $e=(6,2,0)$ and $u=0.2$, with an observation $e_{\mathrm{obs}}=(4,0,0)$, $u_{\mathrm{obs}}=0.33$. Then $e_{\mathrm{post}}=(10,2,0)$ and $u_{\mathrm{post}}=0.14$, below both. The observation removes $\theta=0.29$ of the doubt and moves $P\approx(0.67,\,0.27,\,0.07)$ 29\% of the way to $(1,0,0)$, to $P_{\mathrm{post}}\approx(0.76,\,0.19,\,0.05)$.
\end{exbox}
These identities are specific to cumulative fusion. The reduced Dempster rule of evidential multi-view learning also never raises the uncertainty mass \citep[Prop.~3.3]{han2023tpami}, while averaging fusion, which averages the evidence of the views \citep{xu2024rcml}, can raise it above the smaller of the two.

\subsection{In the state, the lineage is lost}\label{app:lineage}
One could carry the doubt in the imagined state, deducing the opinion on each imagined state from the opinion on the state before, as the deduction of SL does \citep[Ch.~9]{josang2016subjective}.
\begin{keybox}
\begin{proposition}[Deduction forgets the lineage]\label{prop:lineage}
Deduce the opinion on each imagined state from the opinion on the one before, which has doubt $\tilde u_{k-1}$, with $\tilde u_0=0$ at the real state. Take the doubt of the step's conditionals as $\bar u_k$, the only doubt the head reports. Then
\begin{equation*}
  \tilde u_k=\bar u_k+\sigma_k\,\tilde u_{k-1}=\sum_{j=1}^{k}\Big(\prod_{l=j+1}^{k}\sigma_l\Big)\,\bar u_j,\qquad \sigma_k=u^\circ_k-\bar u_k,
\end{equation*}
where $u^\circ_k$ is the doubt that deduction assigns to state $k$ when state $k-1$ is vacuous. So $\tilde u_1=\bar u_1$, $\lvert\tilde u_k-\bar u_k\rvert=\lvert\sigma_k\rvert\,\tilde u_{k-1}\le\lvert\sigma_k\rvert$ at every depth, and the doubt of an earlier step $j$ reaches depth $k$ only through the factor $\prod_{l=j+1}^{k}\sigma_l$: where $\lvert\sigma\rvert$ is small, a doubtful step is forgotten one step later.
\end{proposition}
\end{keybox}
\begin{proof}
Deduction gives the deduced opinion the doubt $u_{Y\Vert X}=u_X\,u_{Y\Vert\hat X}+\sum_iu_{Y|x_i}\,b_X(x_i)$ \citep[Eq.~(9.75)]{josang2016subjective}, where $u_X$ and $b_X$ are the doubt and belief masses of the antecedent and $u_{Y\Vert\hat X}$ is the doubt deduced from a vacuous antecedent. With $u_X=\tilde u_{k-1}$, $u_{Y\Vert\hat X}=u^\circ_k$, $u_{Y|x_i}=\bar u_k$ and $\sum_ib_X(x_i)=1-\tilde u_{k-1}$, this is $\tilde u_k=u^\circ_k\,\tilde u_{k-1}+\bar u_k\,(1-\tilde u_{k-1})$, the first form. Unrolling from $\tilde u_0=0$ gives the second, and $\tilde u_{k-1}\le1$ gives the bound.
\end{proof}
\begin{exbox}
\textbf{Example.} Take a rollout of $H=8$ imagined steps that leaves experience at steps 3 and 4, $\bar u_3=\bar u_4=0.40$, and is familiar elsewhere, $\bar u_k=0.01$; with $\varepsilon=0.01$ the trusts are $0.61$ at steps 3 and 4 and $1$ elsewhere. With $\lvert\sigma_k\rvert\le0.01$, deduction gives states 3 and 4 the doubt $0.40$ and states 5 to 8 the doubt $0.01$, each to within $0.005$: states 5 to 8 are trusted again (Fig.~\ref{fig:trustdepth}a).
\end{exbox}
\paragraph{In words.} On a trained model on cheetah run, $\sigma=0.009$ at the median, and exact deduction gives the same doubt at depth fifteen as at depth one, where compounding would have raised it. The recursion settles after one step at a level set by the current step, and the lineage of the rollout is lost.

\subsection{In the return, the doubt compounds}\label{app:discount}
Carried in the return instead, trust compounds along the rollout.
\begin{keybox}
\begin{proposition}[Carried trust is the trust discounting of SL]\label{prop:discount}
Read the $k$-th imagined step as an advisor in a referral chain that starts at the real state $s_0$, and its trust $\tau_k$ as the projected probability of the binomial opinion that this step is reliable. Then the carried trust $C_n=\prod_{k=1}^{n}\tau_k$ is the projected probability that the transitive trust discounting of SL \citep[Sec.~14.3.4, Def.~14.7]{josang2016subjective} assigns to the chain of the first $n$ steps, and an opinion with belief masses $b$ reported at depth $n$ is discounted to belief $C_n\,b$ and uncertainty $1-C_n\sum_i b_i$. With $\tau\equiv1$ the chain is fully trusted and nothing is discounted.
\end{proposition}
\end{keybox}
\begin{proof}
Transitive discounting multiplies the projected probabilities of the referral edges along the path \citep[Eq.~(14.13)]{josang2016subjective} and scales the belief of the final opinion by that product \citep[Def.~14.7]{josang2016subjective}. With edge $k$ at $\tau_k$ the product is $C_n$. Only the projected probabilities enter, so the result does not depend on how each $\tau_k$ is split into belief, disbelief and uncertainty.
\end{proof}
\begin{exbox}
\textbf{Example.} In the return, the rollout carries the trust $C_n=1$ up to step 2, $0.61$ at step 3 and $0.37$ from step 4 on: states 5 to 8 inherit the doubt of steps 3 and 4, which deduction dropped (Fig.~\ref{fig:trustdepth}a).
\end{exbox}
\paragraph{In words.} Deduction through the state and discounting along the chain are the two ways SL chains opinions. The first forgets a doubtful step once it is passed, the second carries it to every step built on it (Fig.~\ref{fig:trustdepth}a); this is why LucidWM carries doubt as trust in the return.
\begin{figure}[H]
  \centering
  \includegraphics[width=1.000\linewidth]{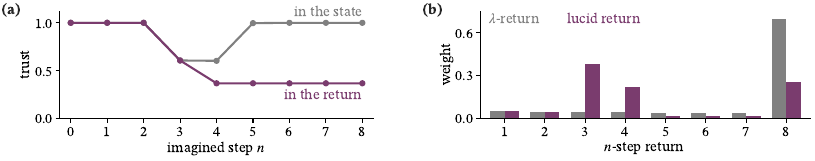}
  \caption{\textbf{Trust in the state loses the lineage; trust in the return compounds}, for the example rollout. (a)~Trust at each imagined step $n$: deduced through the state, $(1-\tilde u_n)/(1-\varepsilon)$ with $\sigma=0.009$ (Prop.~\ref{prop:lineage}), and carried in the return, $C_n$ (Prop.~\ref{prop:discount}). (b)~Weights of the one- to eight-step returns under the $\lambda$-return and the lucid return (Lemma~\ref{lem:unrolled}).}
  \label{fig:trustdepth}
\end{figure}

\subsection{The lucid return keeps the fixed point and discounts model error}\label{app:guarantees}
The lucid return mixes the $n$-step returns; Lemma~\ref{lem:unrolled} gives the mixture, and Thm.~\ref{thm:guar} what it does to learning.
\begin{lemma}[Unrolled lucid return]\label{lem:unrolled}
For every doubt profile, Eq.~\eqref{eq:return} unrolls to Eq.~\eqref{eq:unrolled}: the $n$-step return $R^{(n)}$ has weight $w_n=M_{n-1}-M_n$ for $n<H$ and $w_H=M_{H-1}$. The weights are nonnegative and sum to one, and the returns longer than $n$ steps carry total weight $M_n$. With $\tau\equiv1$, $w_n=(1-\lambda)\lambda^{n-1}$, $w_H=\lambda^{H-1}$ and $R^\tau_t=R^\lambda_t$ (Cor.~\ref{cor:return}).
\end{lemma}
\begin{proof}
Since $R^\tau_H=v(s_H)$, the last step reads $R^\tau_{H-1}=r_{H-1}+\gamma\,v(s_H)$ whatever $\tau_H$. Unrolling Eq.~\eqref{eq:return} from $t=0$,
\begin{equation*}
  R^\tau_0=\sum_{k=0}^{H-1}\gamma^kM_k\,r_k+\sum_{n=1}^{H-1}\gamma^nM_{n-1}(1-\lambda\tau_n)\,v(s_n)+\gamma^HM_{H-1}\,v(s_H).
\end{equation*}
In $\sum_nw_nR^{(n)}$ the reward $r_k$ collects $\sum_{n>k}w_n=M_k$ and $v(s_n)$ collects $w_n$, and $M_{n-1}(1-\lambda\tau_n)=M_{n-1}-M_n$, so the two agree. The weights are nonnegative as $\lambda\tau_n\le1$, telescope to $M_0=1$, and those beyond $n$ sum to $M_n$. With $\tau\equiv1$, $M_n=\lambda^n$.
\end{proof}
\begin{exbox}
\textbf{Example.} With $\lambda=0.95$, the weight of the eight-step return falls from $0.70$ under the $\lambda$-return to $0.26$, and the three- and four-step returns, which hand over to the critic at steps 3 and 4, gain $0.34$ and $0.18$ (Fig.~\ref{fig:trustdepth}b).
\end{exbox}
\begin{keybox}
\begin{thmrestated}[Contraction and bias bound, restated and extended]
For any fixed doubt profile: (i) the critic's update under Eq.~\eqref{eq:return} is a $\gamma$-contraction to $\hat V^{\pi}$, the value of $\pi$ in the imagined model, as under Eq.~\eqref{eq:lambda}; (ii) with the critic at the true value $V^{\pi}$ and $\delta_j=\sup\lvert(\hat T-T)V^{\pi}\rvert$ over the states reached after $j$ imagined steps, where $\hat T$ and $T$ are the Bellman operators of $\pi$ in the model and in the environment,
\begin{equation*}
  \big\lVert\,\mathbb E[R^{\tau}_0]-V^{\pi}\big\rVert_\infty\;\le\;\sum_{j=0}^{H-1}\gamma^{\,j}\lambda^{j}\,C_j\,\delta_j ;
\end{equation*}
(iii) when the trust varies with the state--action pair, (i) holds as it stands, and at each real state $s_0$, with $\eta=(\hat T-T)V^\pi$ the model's one-step error,
\begin{equation*}
  \big\lvert\,\mathbb E[R^{\tau}_0\mid s_0]-V^{\pi}(s_0)\big\rvert\;\le\;\sum_{j=0}^{H-1}\gamma^{\,j}\lambda^{j}\,\mathbb E\big[C_j\,\lvert\eta(s_j)\rvert\;\big|\;s_0\big].
\end{equation*}
\end{thmrestated}
\end{keybox}
\begin{proof}
\textit{Idea.} By Lemma~\ref{lem:unrolled} the target is a mixture of $n$-step returns with weights set by the head; each part then follows from one backward recursion.

\textit{(i).} Let $(\mathcal Rv)(s_0)=\mathbb E[R^\tau_0\mid s_0]$, over imagined rollouts from $s_0$ under $\pi$ and the model. By Lemma~\ref{lem:unrolled}, $R^\tau_0=\sum_{k=0}^{H-1}\gamma^kM_kr_k+\sum_{n=1}^{H}\gamma^nw_n\,v(s_n)$, so
\begin{equation*}
  \lvert\mathcal Rv-\mathcal Rv'\rvert(s_0)\;\le\;\mathbb E\sum_{n=1}^{H}\gamma^nw_n\,\lvert v-v'\rvert(s_n)\;\le\;\gamma\,\lVert v-v'\rVert_\infty ,
\end{equation*}
as $\gamma^n\le\gamma$ and the weights sum to one. For the fixed point, Eq.~\eqref{eq:return} reads $R^\tau_t=r_t+\gamma v(s_{t+1})+\gamma\lambda\tau_{t+1}\big(R^\tau_{t+1}-v(s_{t+1})\big)$. Let $v=\hat V^\pi$ and $\zeta_t=\mathbb E[R^\tau_t\mid s_t]-\hat V^\pi(s_t)$, with $\zeta_H=0$. As $r_t+\gamma\hat V^\pi(s_{t+1})$ has mean $\hat V^\pi(s_t)$ under the model, $\zeta_t=\gamma\lambda\tau_{t+1}\,\mathbb E[\zeta_{t+1}\mid s_t]$, so every $\zeta_t$ vanishes: $\hat V^\pi$ is a fixed point, and by the contraction the only one.

\textit{(ii).} Let $v=V^\pi$ and $\zeta_t=\mathbb E[R^\tau_t\mid s_t]-V^\pi(s_t)$. Now $r_t+\gamma V^\pi(s_{t+1})$ has mean $(\hat TV^\pi)(s_t)=V^\pi(s_t)+\eta(s_t)$, since $TV^\pi=V^\pi$. So $\zeta_t=\eta(s_t)+\gamma\lambda\tau_{t+1}\,\mathbb E[\zeta_{t+1}\mid s_t]$, which unrolls to $\zeta_0=\sum_{j=0}^{H-1}\gamma^jM_j\,\mathbb E[\eta(s_j)\mid s_0]$, and $\lvert\mathbb E[\eta(s_j)\mid s_0]\rvert\le\delta_j$.

\textit{(iii).} Both parts use only that the trust of a step is known at the state and action it leaves. The weights of Lemma~\ref{lem:unrolled} stay nonnegative and sum to one on every rollout, so (i) holds; with $\tau_{t+1}$ inside the expectation, (ii) unrolls to $\zeta_0=\sum_{j}\gamma^j\lambda^j\,\mathbb E[C_j\,\eta(s_j)\mid s_0]$, which gives the bound.
\end{proof}

\needfit{16\baselineskip}
\subsection{When trust tightens the bound}\label{app:tighter}
With a critic off the true value, trust also sets how much of the return rests on the critic.
\begin{keybox}
\begin{proposition}[When lowering trust tightens the bound]\label{prop:critic}
Let the critic err by $\Delta_v=\lVert v-V^{\pi}\rVert_\infty$. For any fixed doubt profile:\\
(i) the bias is bounded by
\begin{equation*}
  \big\lVert\,\mathbb E[R^{\tau}_0]-V^{\pi}\big\rVert_\infty\;\le\;\gamma\Delta_v+\delta_0+\sum_{j=1}^{H-1}\gamma^jM_j\big(\delta_j-(1-\gamma)\Delta_v\big),
\end{equation*}
and lowering the trust $\tau_k$ of one step, $1\le k\le H-1$, with $C_k>0$, lowers this bound exactly when the mean of $\delta_j$ over $j\ge k$, weighted by $\gamma^jM_j$, exceeds $(1-\gamma)\Delta_v$;\\
(ii) with the critic at the true value, $\Delta_v=0$, this is the bound $B(C)=\sum_{j=0}^{H-1}\gamma^{j}\lambda^{j}C_j\delta_j$ of Eq.~\eqref{eq:bias}, and against the standard return, $\tau\equiv1$, $B(\mathbf 1)-B(C)=\sum_{j=0}^{H-1}\gamma^{j}\lambda^{j}(1-C_j)\,\delta_j\ge0$: the error at each depth is discounted by the trust lost before it.
\end{proposition}
\end{keybox}
\begin{proof}
(i) By Lemma~\ref{lem:unrolled}, replacing $V^\pi$ by $v$ changes $R^\tau_0$ by $\sum_{n=1}^{H}\gamma^nw_n\,(v-V^\pi)(s_n)$, at most $\Delta_v\sum_{n=1}^{H}\gamma^nw_n$ in size, and summation by parts gives $\sum_{n=1}^{H}\gamma^nw_n=\gamma-(1-\gamma)\sum_{j=1}^{H-1}\gamma^jM_j$. Adding Thm.~\ref{thm:guar}(ii) gives the bound. Every $M_j$ with $j\ge k$ is proportional to $\tau_k$ and no other term depends on it, so the bound is linear in $\tau_k$, whose slope has the sign of $\sum_{j\ge k}\gamma^jM_j\big(\delta_j-(1-\gamma)\Delta_v\big)$. (ii) Subtract the two sums; as $\bar u\ge\varepsilon$, $\tau\le1$, so $C_j\le1$, and $\delta_j\ge0$.
\end{proof}
\begin{exbox}
\textbf{Example.} In the rollout, errors at depths 0 to 2, up to the first doubtful transition, keep their factors $M_j=1,\,0.95,\,0.90$ in the bound. Errors from depth 3 on, which follow the doubtful steps, enter with $M_3=0.52$ down to $M_7=0.26$, where the $\lambda$-return has $\lambda^j$, $0.86$ down to $0.70$.
\end{exbox}
\paragraph{In words.} Every $M_j$ with $j\ge k$ carries $\tau_k$, so one doubtful step discounts every model error that follows it and none before it; a smaller $\lambda$ would discount every depth alike. As the doubt rises where the rollout leaves experience (Sec.~\ref{sec:experience}), these are the errors made after it has left what the model knows. Prop.~\ref{prop:critic}(i) says where the discount pays when the critic errs: where the model's one-step error beyond a step exceeds $(1-\gamma)\Delta_v$, the critic's error per step. Which imagined steps the trust actually removes is measured in Sec.~\ref{sec:imagination}.

\subsection{Doubt ranks actions before acting}\label{app:ordering}
\begin{keybox}
\begin{proposition}[Doubt before acting]\label{prop:acting}
The doubt $u(s_t,a)$ is computed by the head from $(s_t,a)$ alone, so it is known for every candidate action before any is taken. Above the floor $\varepsilon$, $u(s_t,a)>u(s_t,a')$ exactly when $S(s_t,a)<S(s_t,a')$, for every prior weight $W>0$.
\end{proposition}
\end{keybox}
\begin{proof}
The evidence $e=\mathrm{softplus}\,\ell(s_t,a)$ needs no observation of the outcome, and $u=W/(W+S)$ falls strictly with $S$ for every $W>0$.
\end{proof}
What acting on trust changes in the agent's decisions is measured in Sec.~\ref{sec:decision}.

\FloatBarrier
\appgroup{Experiments}
\section{Experimental details}\label{app:exp}
LucidWM is compared with seventeen readouts on four bases and three environment families (Table~\ref{tab:ledger}). This appendix gives the bases, readouts and protocols; App.~\ref{app:further} gives the results beyond the main text.

\subsection{Protocols}\label{app:protocol}
\textbf{Maze.} The agent trains on a ViZDoom map with one room sealed off. The \emph{opened} map unseals it; the \emph{repainted} map keeps every wall in place and changes the paint of one room. The repainted dead end of Sec.~\ref{sec:decision} keeps every wall and repaints the deep half of one dead end. A walk is a scripted route that reads the agent's position only, never the model, so that every readout sees the same frames. Learning to doubt leaves the task intact: every trained agent, with LucidWM and without, reaches the goal in all 100 test episodes.

\textbf{Corrupted actions.} From real start states the model imagines forward while each action is replaced, with a set probability, by a random one drawn uniformly from the action space. The lift of a readout is its value under fully corrupted actions over its value under the true ones.

\textbf{Alarm line.} Each readout is set against its own alarm line, taken on familiar frames of the same run; Figs.~\ref{fig:walkin}, \ref{fig:falls}, \ref{fig:night} and~\ref{fig:fourbases} show each readout in alarm units, $0$ at its familiar level and $1$ at its line.

\subsection{Four bases}\label{app:results}
Table~\ref{tab:bases} lists the four bases.
\begin{table}[H]
  \caption{The four bases; $\varepsilon=0.01$ and $\lambda=0.95$ on all. $G\times K$: categorical variables and classes of the latent state; $W$: prior weight; $H$: imagination horizon.}
  \label{tab:bases}
  \centering\footnotesize\setlength{\tabcolsep}{6pt}\renewcommand{\arraystretch}{1.167}
  \begin{tabular*}{\linewidth}{@{\extracolsep{\fill}}llccc@{}}
    \toprule
    base & dynamics & $G\times K$ & $W$ & $H$\\
    \midrule
    DreamerV3 & recurrent, pixel decoder & $32\times32$ & 2 & 15\\
    R2-Dreamer & recurrent, no reconstruction & $32\times16$ & 1 & 15\\
    EMERALD & masked transformer, spatial & $4\times4$ cells of $32\times32$ & 2 & 15\\
    OC-STORM & transformer & $32\times32$ & 2 & 16\\
    \bottomrule
  \end{tabular*}
\end{table}

\subsection{Every experiment at a glance}\label{app:ledger}
\par{\footnotesize\setlength{\tabcolsep}{3pt}\renewcommand{\arraystretch}{1.167}\setlength{\LTpre}{\intextsep}\setlength{\LTpost}{\intextsep}\setlength{\LTcapwidth}{\linewidth}\captionsetup{font=normalsize}
\begin{longtable}{@{}>{\raggedright\arraybackslash}p{3.77cm}>{\raggedright\arraybackslash}p{2.61cm}>{\raggedright\arraybackslash}p{2.85cm}>{\raggedright\arraybackslash}p{5.29cm}@{}}
\caption{Every experiment in the paper and what it finds. DMC in Table~\ref{tab:main}(B): walker, cheetah, pendulum.}\label{tab:ledger}\\
\toprule
experiment & bases & environments & result\\
\midrule
\endfirsthead
\toprule
experiment & bases & environments & result\\
\midrule
\endhead
\bottomrule
\endlastfoot
    \multicolumn{4}{@{}l}{\emph{main text}}\\*
    after observation, Tab.~\ref{tab:main}(A) & all four & maze & first on every base\\
    before decision, Tab.~\ref{tab:main}(B) & all four & DMC, Crafter & first on all 16 rows\\
    leaving experience, Fig.~\ref{fig:cloud} & DreamerV3 & walker & 9 in 10 over the line; base flat\\
    two changed rooms, Fig.~\ref{fig:maze} & DreamerV3 & maze & alarm in both; base in neither\\
    learning, Fig.~\ref{fig:growth} & DreamerV3 & cheetah & doubt falls with the error\\
    trust as a filter, Fig.~\ref{fig:riskcov} & DreamerV3 & walker, cheetah, finger & gain within depth on all three, up to 23\%\\
    acting on trust, Fig.~\ref{fig:decision} & DreamerV3 & maze & goal at step 190; base at 362\\
    \addlinespace[2pt]
    \multicolumn{4}{@{}l}{\emph{appendix}}\\*
    new room, Fig.~\ref{fig:walkin} & DreamerV3 & maze & 25/29 walks; base 0/29\\
    one pass vs.\ three models, Fig.~\ref{fig:ensemble} & DreamerV3 & maze & 8/10 walks; ensembles 0/10\\
    falls, Fig.~\ref{fig:falls} & DreamerV3 & cheetah & 20/20 before onset; entropy 0/20\\
    nightfall, Fig.~\ref{fig:night} & DreamerV3 & Crafter & peak $2.6\times$ its line; base $<0.5\times$\\
    every base, Fig.~\ref{fig:fourbases} & all four & maze & alarm on all four bases\\
    corruption ladder, Fig.~\ref{fig:corruption} & OC-STORM, EMERALD & walker, cheetah, Crafter & over fivefold; beats 20-pass readouts\\
    any line, any checkpoint, Fig.~\ref{fig:robust} & DreamerV3 & maze & 28/28 rules, 38/38 checkpoints\\
    six checkpoints, Fig.~\ref{fig:growthall} & DreamerV3 & cheetah & doubt halved; error $24.6\times$ lower\\
    dream and reality, Fig.~\ref{fig:dreams} & DreamerV3 & walker, Crafter & trust 0.16 against 0.42 when calm\\
    trust over ten runs, Fig.~\ref{fig:trust10}(a) & DreamerV3 & four DMC tasks & 11\% within depth, first of six\\
    hyperparameters, Fig.~\ref{fig:trust10}(b,\,c) & DreamerV3 & cheetah, walker & whole rollout best; return flat\\
    choosing by trust, Fig.~\ref{fig:crafter}(a) & DreamerV3 & Crafter & deaths/1k steps 3.80 against 4.08\\
    learning Crafter, Fig.~\ref{fig:crafter}(b,\,c) & DreamerV3 & Crafter & score 15.6\% against 9.0\%\\
    where vetoes fire, Fig.~\ref{fig:vetomap} & DreamerV3 & maze & 143/145 in a changed room\\*
    \midrule
    \multicolumn{4}{@{}l}{\emph{total:} 210 compared cells in Table~\ref{tab:main} $\cdot$ 4 bases $\cdot$ 3 environment families, 7 tasks $\cdot$ 17 readouts}\\
\end{longtable}}

\subsection{Seventeen readouts and their cost}\label{app:opponents}
Every readout is read on the same frames and imagined futures, and oriented so that higher means less familiar (Table~\ref{tab:readouts}). The fitted heads and multiple forwards are RND \citep{burda2019rnd}, an evidential head \citep{sensoy2018edl}, latent disagreement \citep{sekar2020plan2explore}, the Mahalanobis distance \citep{lee2018mahalanobis}, deep nearest neighbours \citep{sun2022knn}, deep ensembles \citep{lakshminarayanan2017ensembles}, MC dropout \citep{gal2016dropout}, a Laplace posterior \citep{daxberger2021laplace} and snapshot ensembles \citep{huang2017snapshot}; \emph{self}, three samples of one model, is our control for \emph{deep}.
\par{\footnotesize\setlength{\tabcolsep}{4pt}\renewcommand{\arraystretch}{1.167}\setlength{\LTpre}{\intextsep}\setlength{\LTpost}{\intextsep}\setlength{\LTcapwidth}{\linewidth}\captionsetup{font=normalsize}
\begin{xltabular}{\linewidth}{@{}ll>{\raggedright\arraybackslash}Xl@{}}
\caption{The readouts compared in Table~\ref{tab:main}, as defined on DreamerV3; the other bases read them from their own features. Cost as in Table~\ref{tab:main}, for tests (A) / (B).}\label{tab:readouts}\\
\toprule
family & readout & what it reads & cost (A / B)\\
\midrule
\endfirsthead
\toprule
family & readout & what it reads & cost (A / B)\\
\midrule
\endhead
\bottomrule
\endlastfoot
    free & base & our doubt, read from the logits of the unmodified base & $1\times$\\*
     & entropy & entropy of the latent distribution & $1\times$\\*
     & max p. & one minus its largest probability & $1\times$\\*
     & KL & divergence of the posterior from the prediction & $1\times$\\*
     & recon. & reconstruction error of the arriving frame & $1\times$\\*
     & 1-step & error of the arriving frame decoded from the prediction & $1\times$\\
    \addlinespace[2pt]
    fitted heads & RND & error against a fixed random network, on the model state & $+.18$ / $+.13$\\*
     & RND-e & the same, on the encoder embedding & $+.38$\\*
     & evid. & vacuity of an evidential head trained on the model state & $+.12$ / $+.10$\\*
     & latent & disagreement of five transition heads & $+.62$ / $+.52$\\*
     & Mahal. & Mahalanobis distance to the training states & $+.04$\\*
     & kNN & distance to the fiftieth nearest training state & $+2.4$\\
    \addlinespace[2pt]
    multiple forwards & self & variance of three sampled predictions of one model & 3 fwd\\*
     & deep & disagreement of three independently trained models & $3\times$\\*
     & MC drop & variance over twenty dropout masks on the transition head & 20 fwd\\*
     & Laplace & variance over twenty weight samples of a Laplace posterior & 20 fwd\\*
     & snap. & disagreement of four earlier snapshots of the model & 4 fwd\\
\end{xltabular}}

\FloatBarrier
\section{Further results}\label{app:further}
\setlength{\intextsep}{10pt plus 2pt minus 2pt}

\subsection{The alarm fires in the new room}\label{app:newroom}
\begin{figure}[H]
  \centering
  \includegraphics[scale=0.937]{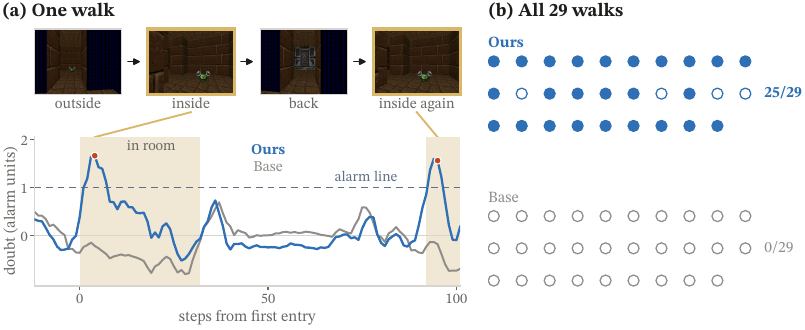}
  \caption{\textbf{Into the new room, and into it again.} (a)~One walk: views outside, inside (gold), back and inside again, and the doubt of Ours and Base in alarm units; sand: in the room; red: alarms. (b)~Walks on which each readout raises its alarm inside the room.}
  \label{fig:walkin}
\end{figure}
The room sealed in training shows familiar walls in a layout the agent has never experienced. The doubt raises its alarm inside it on 25 of 29 walks and the base's readout on none (Fig.~\ref{fig:walkin}); along one walk, it peaks each time the agent enters.
\begin{figure}[H]
  \centering
  \includegraphics[scale=0.937]{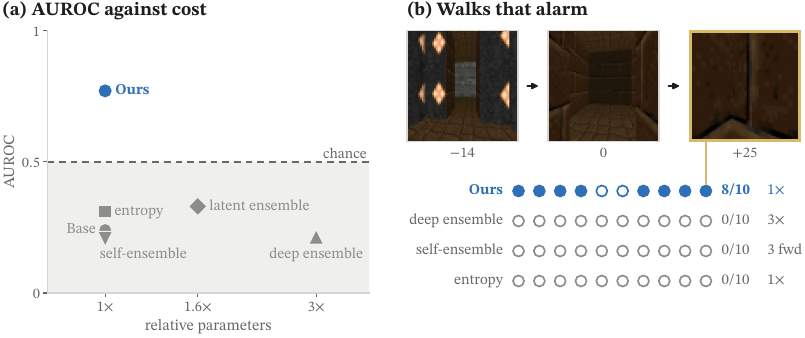}
  \caption{\textbf{One pass against three models.} (a)~AUROC of the new room against relative parameters, as in Table~\ref{tab:main}(A); grey band: below chance, where a readout reads the new room as more familiar than the maze. (b)~Walks through the opened door on which each readout raises its alarm, with its cost; top: walk~10 at $-14$, $0$ and $+25$ steps from entry, gold where Ours raises its alarm, joined to its dot.}
  \label{fig:ensemble}
\end{figure}
One pass does what three models cannot (Fig.~\ref{fig:ensemble}): at the cost of one forward pass the doubt ranks the new room above the maze, while the base's latent and deep ensembles, at up to three times the cost, read it as \emph{more} familiar; on ten further walks through the door, the doubt raises its alarm on eight, the deep and self-ensembles and the entropy on none.

\subsection{The alarm fires before the fall}\label{app:falls}
\begin{figure}[H]
  \centering
  \includegraphics[scale=0.937]{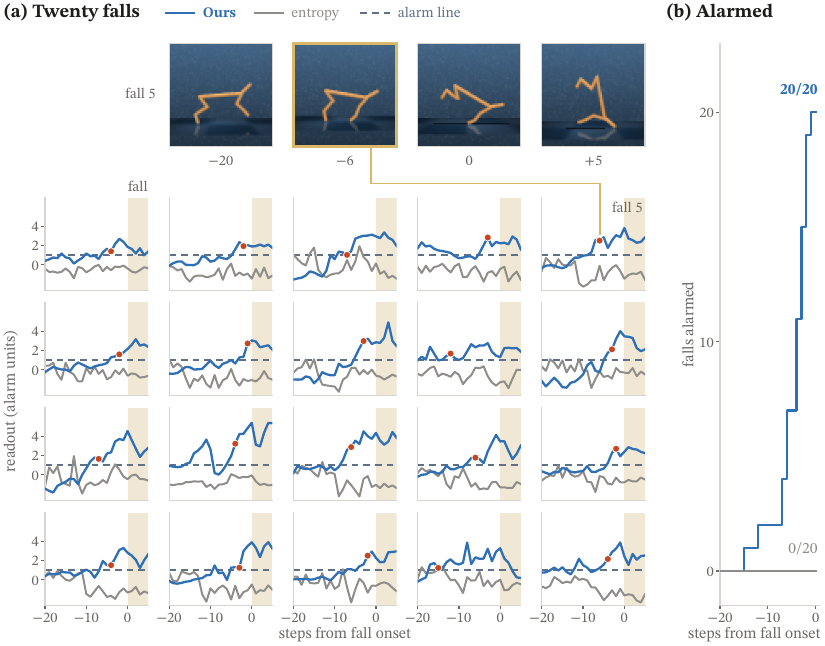}
  \caption{\textbf{Twenty falls, twenty alarms before the fall.} (a)~Doubt (blue) and the entropy of the same prediction (grey) around the onset of each of twenty consecutive falls, in alarm units; dashed: alarm line; red: onset of the doubt's alarm that holds into the fall; sand: the fall. (b)~Falls whose alarm has fired by each step. Top: fall~5, the fall of Fig.~\ref{fig:teaser}, at $-20$, $-6$, $0$ and $+5$ steps; gold: step $-6$, where its alarm holds; the doubt first crosses the line at $-8$, which Fig.~\ref{fig:teaser} counts as its lead.}
  \label{fig:falls}
\end{figure}
The doubt warns of each of twenty consecutive cheetah falls before it begins (Fig.~\ref{fig:falls}): its alarm fires a median of four steps ahead, while the entropy of the same prediction raises none.

\subsection{The alarm fires at nightfall}\label{app:night}
\begin{figure}[H]
  \centering
  \includegraphics[scale=0.937]{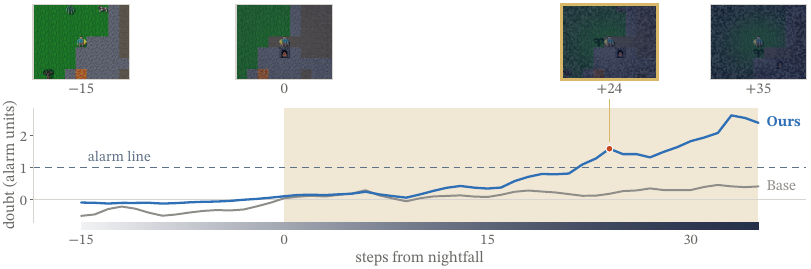}
  \caption{\textbf{Nightfall.} Views from fifteen steps before nightfall to thirty-five after (gold: at the alarm), above the doubt of Ours and Base in alarm units; sand: the night; strip: the fading light.}
  \label{fig:night}
\end{figure}
Nightfall darkens Crafter step by step. The doubt climbs over its line within twenty-five steps and peaks above twice it, while the base's readout never reaches half of its own (Fig.~\ref{fig:night}).

\subsection{The alarm fires on every base}\label{app:everybase}
\begin{figure}[H]
  \centering
  \includegraphics[scale=0.937]{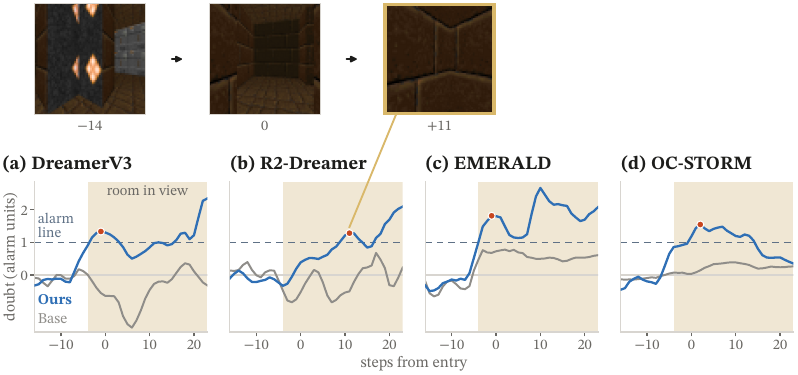}
  \caption{\textbf{Four bases, four alarms.} Doubt of Ours and Base on each base as the agent approaches the new room, in alarm units; sand: the room in view; red: the first alarm of Ours. Top: the route plotted in (b) and (d) at $-14$, $0$ and $+11$ steps from entry, gold at the alarm of (b), joined to its dot; (a) and (c) pass the same door.}
  \label{fig:fourbases}
\end{figure}
On all four bases the doubt crosses its line at the new room (Fig.~\ref{fig:fourbases}).

\subsection{The worse the actions, the higher the doubt}\label{app:ladder}
The more actions are corrupted, the higher the doubt climbs (Fig.~\ref{fig:corruption}). On OC-STORM its median over starts rises over fivefold on walker and over twofold on Crafter, while the base's readout barely moves. On EMERALD walker and cheetah the doubt rises over threefold from one pass, where MC dropout, Laplace and a snapshot ensemble, at up to twenty passes, fall instead.
\begin{figure}[H]
  \centering
  \includegraphics[scale=0.937]{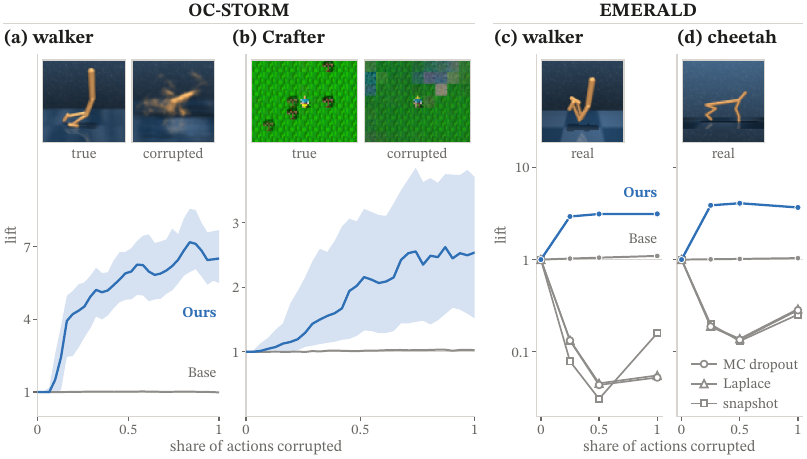}
  \caption{\textbf{The worse the actions, the higher the doubt.} Lift against the share of actions corrupted; thin line at one: no response. (a,\,b)~OC-STORM, median over 300 starts (band: interquartile range of Ours); insets: imagination under true and under corrupted actions. (c,\,d)~EMERALD against MC dropout, Laplace and a snapshot ensemble, on a log scale; insets: a real frame of each task.}
  \label{fig:corruption}
\end{figure}

\subsection{The alarm fires at any line, any checkpoint}\label{app:robust}
\begin{figure}[H]
  \centering
  \includegraphics[scale=0.937]{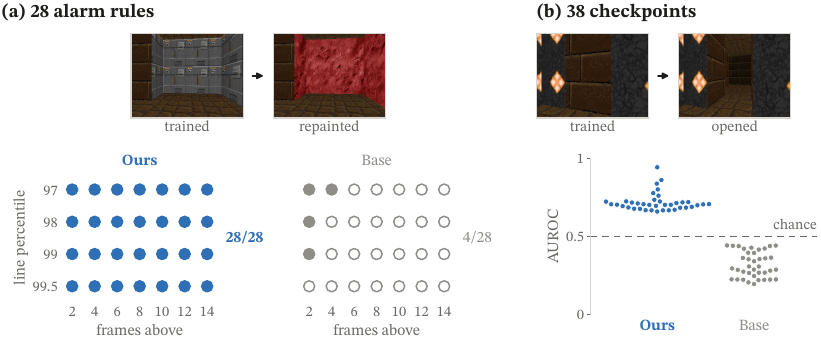}
  \caption{\textbf{Any alarm line, any checkpoint.} (a)~Rules, placement of the line (rows) by frames required above it (columns), under which each readout raises its alarm in the repainted room. (b)~AUROC of the opened room at each of 38 checkpoints of training; dashed: chance. Top: each change seen from one viewpoint, as trained and as changed.}
  \label{fig:robust}
\end{figure}
The alarm does not hinge on where its line is drawn or on when the model is read (Fig.~\ref{fig:robust}). In the repainted room the doubt raises its alarm under all 28 rules, four placements of the line by seven required durations, and the base's readout under only 4. Across 38 checkpoints of training, from $0.1$M to $3.8$M frames, the doubt ranks the frames of the opened room above the familiar ones at every checkpoint, and the base's readout ranks them below chance at every one.

\subsection{The doubt drains as the model learns}\label{app:growth}
At one step, and averaged over all 24 imagined steps, the doubt falls at every checkpoint of the run of Sec.~\ref{sec:calibration} (Fig.~\ref{fig:growthall}): from $0.194$ at 48k environment steps to $0.093$ at 180k, and from $0.181$ to $0.096$. Over the same span the error falls $24.6$-fold at one step and $3.9$-fold on average, while the entropy of the same prediction, averaged likewise, rises from $0.79$ to $1.01$.
\begin{figure}[H]
  \centering
  \includegraphics[scale=0.937]{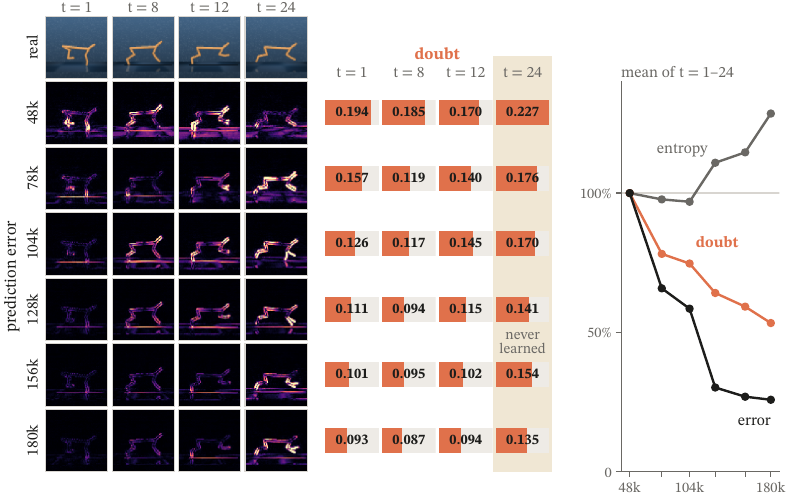}
  \caption{\textbf{Six checkpoints, one future.} As Fig.~\ref{fig:growth}, at all six checkpoints of the run. Right: error, doubt and the entropy of the same prediction, each averaged over the 24 steps, relative to the first checkpoint.}
  \label{fig:growthall}
\end{figure}

\subsection{Trust fades faster where the dream drifts}\label{app:filmstrip}
Fed one real action sequence, the world and the imagination part ways (Fig.~\ref{fig:dreams}). On walker the imagined walker falls while the real one walks on, and the pixel error between the two futures grows almost fivefold. On Crafter lava appears in the real world from the second step and never in the dream, and the trust of this rollout drains far faster than that of rollouts started in calm scenes: $0.16$ against $0.42$ at step 6.
\begin{figure}[H]
  \centering
  \includegraphics[scale=0.937]{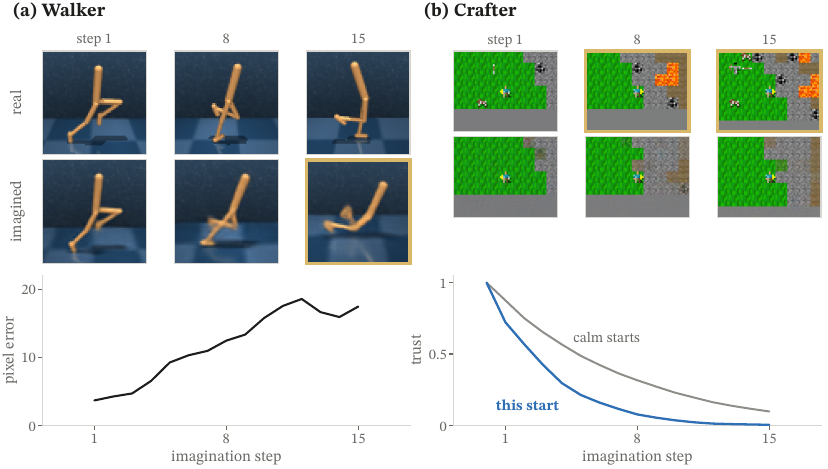}
  \caption{\textbf{Trust fades faster where the dream drifts.} Real (top) and imagined (bottom) futures under one action sequence. (a)~Walker; gold: the imagined walker has fallen; the pixel error between the two futures. (b)~Crafter; gold: lava in the real world; trust of this rollout against the median over thirty calm starts.}
  \label{fig:dreams}
\end{figure}

\subsection{Trust keeps its edge across ten runs}\label{app:trust10}
Within each depth, averaged over ten runs across walker, cheetah, finger and cartpole, trust removes 11.0\% of the error, more than any other readout of Sec.~\ref{sec:imagination} (Fig.~\ref{fig:trust10}a): trust compounded over two steps removes 9.3\%, the doubt of one step 6.8\%, the doubt of the posterior 4.2\%, entropy 3.6\% and depth alone 0.0\%. Trust removes error within depth on each of the ten runs.
\begin{figure}[H]
  \centering
  \includegraphics[scale=0.937]{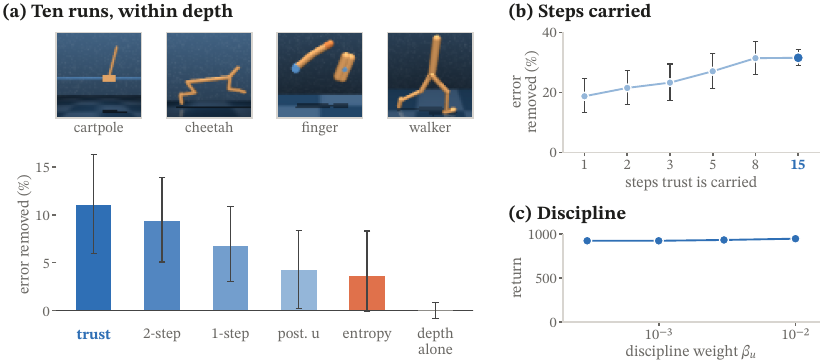}
  \caption{\textbf{Trust across ten runs, and two hyperparameters.} (a)~Error removed by dropping the worst fifth of imagined steps, ranked within each depth as in the solid bars of Fig.~\ref{fig:riskcov}, averaged over ten runs across walker, cheetah, finger and cartpole; top: the four tasks. (b)~The same, ranked over all depths on cheetah from pixels as in Fig.~\ref{fig:riskcov}b, with trust compounded over the last $w$ steps; $w=15$: the whole rollout, as in LucidWM. (c)~Return on walker against the discipline weight $\beta_u$. Whiskers: bootstrap 95\% intervals.}
  \label{fig:trust10}
\end{figure}

\subsection{Hyperparameter analysis}\label{app:window}
LucidWM carries trust through the whole rollout, and no shorter window $w$ does better (Fig.~\ref{fig:trust10}b): on cheetah from pixels, the task of Fig.~\ref{fig:riskcov}b, the error removed grows with every longer window, from 18.7\% at $w=1$, the doubt of one step, to 31.5\% at all 15 steps. The weight $\beta_u$ of the evidence discipline leaves the return in place (Fig.~\ref{fig:trust10}c): on walker, from $3\times10^{-4}$ to $10^{-2}$, it stays between 924 and 949.

\subsection{Choosing futures by trust cuts deaths}\label{app:choose}
\begin{figure}[H]
  \centering
  \includegraphics[scale=0.937]{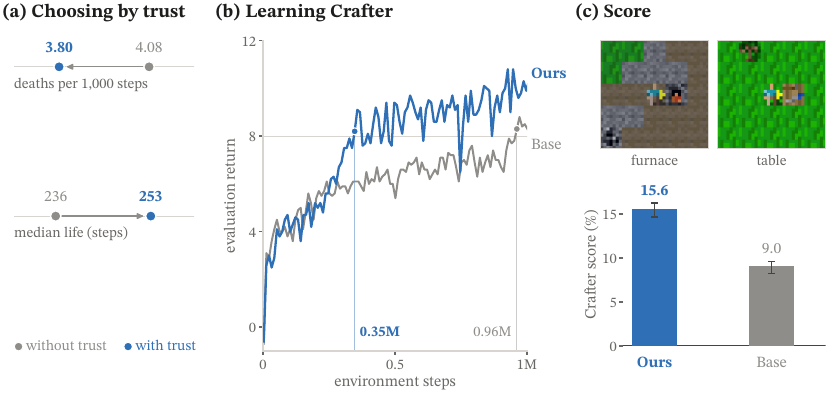}
  \caption{\textbf{Fewer deaths, faster learning.} (a)~One agent choosing among imagined futures scored with and without trust; median life: per seed, averaged over 50 seeds. (b)~Evaluation return over training; dots and lines: where each first reaches 8. (c)~The official Crafter score of the 22 achievements over the last 100 evaluation episodes; whiskers: bootstrap 95\% intervals over episodes; tiles: from the scored episodes, Ours placing a furnace and Base a table.}
  \label{fig:crafter}
\end{figure}
Used beyond the veto to choose among whole futures, trust keeps the agent alive longer (Fig.~\ref{fig:crafter}a). At every step in Crafter, LucidWM on DreamerV3 imagines eight futures of 15 steps with its own actor and takes the first action of the future that scores highest. With trust, a future is scored by a variant of $J$ (Eq.~\eqref{eq:gate}) that keeps $\lambda$, weighting step $k$ by $\lambda^{k}C_k$, the share of the lucid return that reaches it; without trust, by $\lambda^{k}$ alone. Both scores see the same futures, and trust changes the choice at 43\% of steps. Over 50 seeds of 1{,}000 steps, trust cuts deaths from 4.08 to 3.80 per 1{,}000 steps and lengthens the median life from 236 to 253 steps (paired Wilcoxon, $p=0.041$ and $0.0016$).

\subsection{LucidWM learns Crafter faster than its base}\label{app:crafterlearn}
The agent of App.~\ref{app:choose} learns Crafter faster than its base and scores higher (Fig.~\ref{fig:crafter}b,\,c). After 0.35M steps it reaches the evaluation return of 8 that its base first reaches after 0.96M, and over the last 100 evaluation episodes of training its Crafter score is 15.6\% against 9.0\%, with 11.0 achievements unlocked per episode against 8.9.

\subsection{Vetoes fire where the world changed}\label{app:vetoes}
\begin{figure}[H]
  \centering
  \includegraphics[scale=0.937]{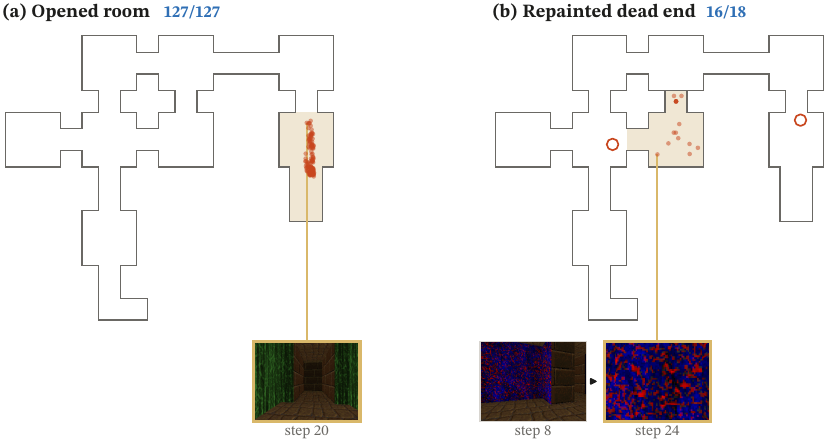}
  \caption{\textbf{Vetoes fire where the world changed.} Where every veto fired on (a)~the opened map and (b)~the repainted dead end; sand: the changed rooms; rings: vetoes outside them. Insets: the agent's view at a veto, gold and joined to its dot; in (b), also the repainted room at step~8. The veto at step~24 in (b) is that of Fig.~\ref{fig:decision}.}
  \label{fig:vetomap}
\end{figure}
The veto of Sec.~\ref{sec:decision} fires where the world has changed (Fig.~\ref{fig:vetomap}). On the two altered maps, 143 of its 145 vetoes fall inside a changed room: all 127 in the opened room and 16 of 18 in the repainted dead end.

\end{document}